\documentclass{article}

\usepackage[preprint]{mystyle}

\usepackage[utf8]{inputenc} % allow utf-8 input
\usepackage[T1]{fontenc}    % use 8-bit T1 fonts
\usepackage{hyperref}       % hyperlinks
\usepackage{url}            % simple URL typesetting
\usepackage{amsfonts}       % blackboard math symbols
\usepackage{microtype}      % microtypography
\usepackage{xcolor}         % colors
\usepackage{tabularx,eucal}
\usepackage{array}

\hypersetup{
    colorlinks=true,
    linkcolor=blue!60,
    citecolor=blue!60,
    urlcolor=blue!60,
    filecolor=blue!60,
}

\usepackage{amsmath,amssymb,amsthm}
\usepackage{algorithm}
\usepackage{algorithmic}
\usepackage{graphicx}
\usepackage{booktabs}
\usepackage{colortbl,xcolor}
\usepackage{soul}
\usepackage{multirow}
\usepackage{enumitem}

\colorlet{llgray}{lightgray!40}
\sethlcolor{llgray}

\newtheorem{theorem}{Theorem}

\newtheorem{proposition}[theorem]{Proposition}

\title{MARBLE: Multi-Aspect Reward Balance for Diffusion RL}

\newcommand{\ours}{{\sc Marble}}
\newcommand{\oursbf}{{\textbf{\textsc{Marble}}}}

\author{Canyu Zhao$^{1}$\thanks{Work done during an internship at HiThink.}
~~
Hao Chen$^{1}$
~~
Yunze Tong$^{1}$
~~
Yu Qiao$^{2}$
~~
Jiacheng Li$^{2}$
~~
Chunhua Shen$^{1,3}$
\\
$^1$ Zhejiang University~~
$^2$ HiThink~~
$^3$ Zhejiang University of Technology
}

\begin{document}

\maketitle

\begin{abstract}
Reinforcement learning (RL) fine-tuning has become the dominant approach for aligning diffusion models with human preferences. However, assessing images is intrinsically a multi-dimensional task, and multiple evaluation criteria need to be optimized simultaneously. Existing practice deal with multiple rewards by training one specialist model per reward, optimizing a weighted-sum reward $R(x){=}\sum_k w_k R_k(x)$, or sequentially fine-tuning with a hand-crafted stage schedule. These approaches either fail to produce a unified model that can be jointly trained on all rewards or necessitates heavy manually tuned sequential training.
We find that the failure stems from using a naive weighted-sum reward aggregation. This approach suffers from a sample-level mismatch because most rollouts are \textit{specialist samples}, highly informative for certain reward dimensions but can be irrelevant for others; consequently, weighted summation dilutes their supervision.
To 
address this issue, we propose \oursbf{} (\textbf{M}ulti-\textbf{A}spect \textbf{R}eward \textbf{B}a\textbf{L}anc\textbf{E}), a gradient-space 
optimization 
framework that maintains  independent advantage estimators for each reward, computes per-reward policy gradients, and harmonizes them into a single update direction without manually-tuned  reward weighting, by solving a Quadratic Programming problem. We further propose an amortized formulation that exploits the affine structure of the loss used in DiffusionNFT,  
to reduce the per-step cost from $K{+}1$ backward passes to near single-reward baseline cost, together with EMA smoothing on the balancing coefficients to stabilize updates against transient single-batch  fluctuations. On SD3.5 Medium with five rewards, \ours{} improves all five reward dimensions \textit{simultaneously}, turns the worst-aligned reward's gradient cosine from negative under weighted summation in $80\%$ of mini-batches to consistently positive, and runs at $0.97\times$ the training speed of baseline training. Homepage and code repo: \href{http://aim-uofa.github.io/MARBLE}{HERE}. 
\end{abstract}

\begin{figure}[t]
    \centering
    \includegraphics[width=\textwidth]{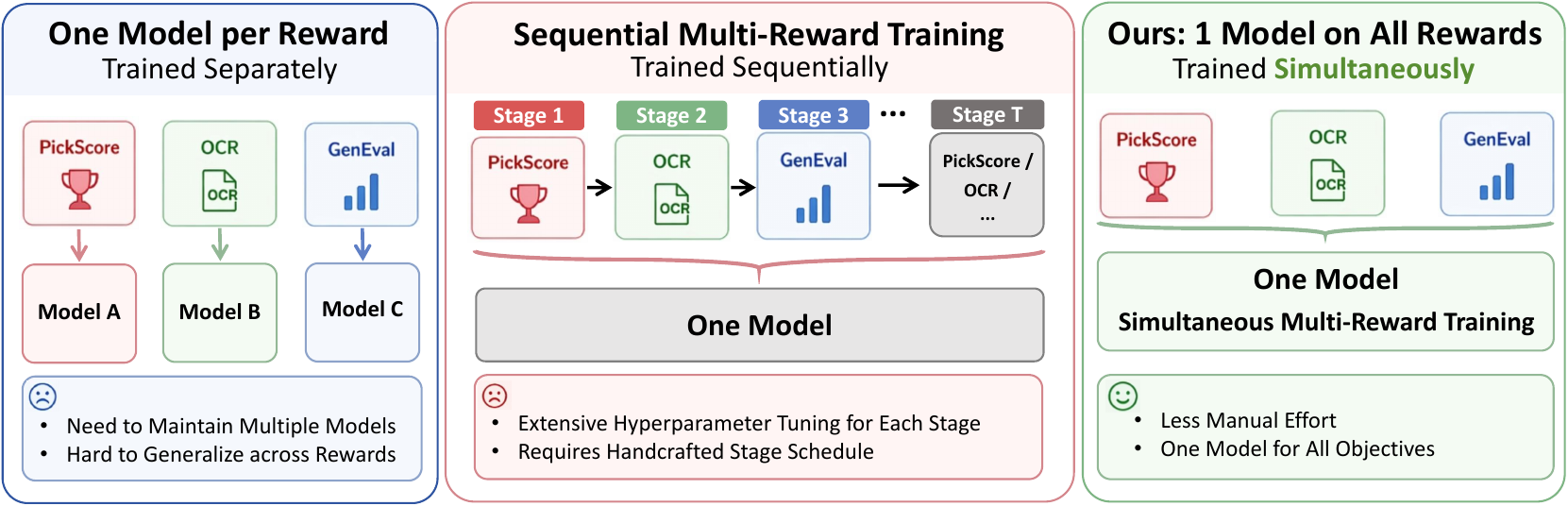}
    \vspace{-3mm}
    \caption{Comparison of multi-reward training paradigms. \textbf{Left:} Training one model per reward requires maintaining multiple models and cannot generalize across reward dimensions. \textbf{Middle:} Sequential multi-reward training produces a single model but demands extensive hyperparameter tuning and handcrafted stage schedules. \textbf{Right:} \ours{} trains a single model on all rewards simultaneously with minimal manual effort.}
    \label{fig:teaser}
    \vspace{-4mm}
\end{figure}

\section{Introduction}
\label{sec:intro}

Reinforcement learning (RL) fine-tuning has emerged as the dominant paradigm for aligning diffusion model outputs with human preferences, yielding notable improvements in aesthetic quality, text-image alignment, and compositional accuracy~\citep{liu2025flow,zhang2026op,tong2026alleviating,zheng2026DiffusionNFT}. In practice, however, generation quality is inherently \emph{multi-dimensional}. A high-quality image should simultaneously exhibit aesthetic appeal, faithfulness to the text prompt, and fine-grained correctness such as accurate text rendering and coherent object placement. These aspects are difficult to optimize jointly. 
Existing methods typically optimize a separate model for each individual reward~\citep{liu2025flow,zhang2026op,tong2026alleviating},
or sequentially fine-tune a single model on different reward datasets
\citep{zheng2026DiffusionNFT}.
However, the former does not yield a unified model, while the latter relies on substantial manual effort in designing the training schedule and hyperparameters. For example, DiffusionNFT~\citep{zheng2026DiffusionNFT} uses a hand-crafted sequence of stages: 800 iterations on reward 1, followed by 300 iterations on reward 2; 200 iterations on reward 1; 200 iterations on reward 2, and finally 100 iterations on reward 3, which requires substantial manual tuning and suffer from forgetting previously acquired rewards.

Therefore, \textit{the central challenge lies in developing a principled approach to conveniently and effectively optimize a single model across multiple reward objectives while eliminating heuristic manual tuning.}
A natural approach to multi-reward optimization is to combine all reward signals into a single scalar objective, typically via a weighted sum
$R(x)=\sum_k w_k R_k(x)$.  However, in practice, directly optimizing a diffusion model with this naively aggregated reward often results in performance degradation rather than improvement.
We trace the failure of scalar aggregation to a sample-level mismatch that we call the \emph{specialist sample} phenomenon (Figure~\ref{fig:specialist}). Many rollouts are informative for only a part of reward dimensions and uninformative or even inapplicable for the rest. For example, an image of a cat carries no signal for OCR-related rewards, and a generation with strong text rendering may be only average aesthetically. Under $R(x){=}\sum_k w_k R_k(x)$, the value of such a sample is diluted by the unrelated dimensions, and the resulting advantage no longer reflects the dimension on which the sample is genuinely useful. We further empirically confirm this dilution at the gradient level (Section~\ref{sec:failure}): the weighted-sum update direction is anti-aligned with single reward gradient, meaning the update actively pushes against some reward most of the time. 

To address this problem, we propose \oursbf{}, a gradient-space reward balance framework that preserves reward-specific supervision throughout optimization. Rather than collapsing rewards into a scalar, \ours{} maintains an independent advantage estimator per reward so that each sample is credited precisely on the dimensions for which it is informative, computes per-reward policy gradients, normalizes them to remove scale disparities, and harmonizes them into a single update direction. 
To ensure scalability during training, we develop an amortized formulation that leverages the affine structure of the DiffusionNFT loss, thereby reducing the per-step computational cost to nearly that of a single-reward baseline.
Also, we apply EMA smoothing on the balancing coefficients so that certain rewards are not transiently silenced when a single mini-batch happens to carry weak signal for them.
In summary, our contributions are:
\begin{itemize}
    \item \textbf{We characterize the \emph{specialist sample} problem in multi-reward diffusion RL.} Across rollouts on SD3.5 Medium, weighted-sum aggregation produces an update direction that is anti-aligned with at least one reward's gradient in $80\%$ of mini-batches, formally quantifying why scalar reward aggregation fails when reward signals are sample-sparse.

    \item \textbf{We propose \oursbf{}, a gradient-space reward balancing framework.} \ours{} combines (i) per-reward advantage decomposition with normalize-and-rescale gradient harmonization, (ii) an amortized variant that reduces multi-reward training cost to near a single-reward baseline by exploiting the affine structure of the DiffusionNFT loss, and (iii) EMA coefficient smoothing that stabilizes amortized balancing weights against transient single-batch fluctuations.

    \item \textbf{\ours{} simultaneously improves all rewards with a single model.} To the best of our knowledge, we are the first to address reward balancing in multi-reward diffusion RL. We believe \ours{} provides a useful foundation for future work on scalable multi-objective alignment of generative models.

\end{itemize}

\begin{figure}[tb!]
    \centering
    \includegraphics[width=\linewidth]{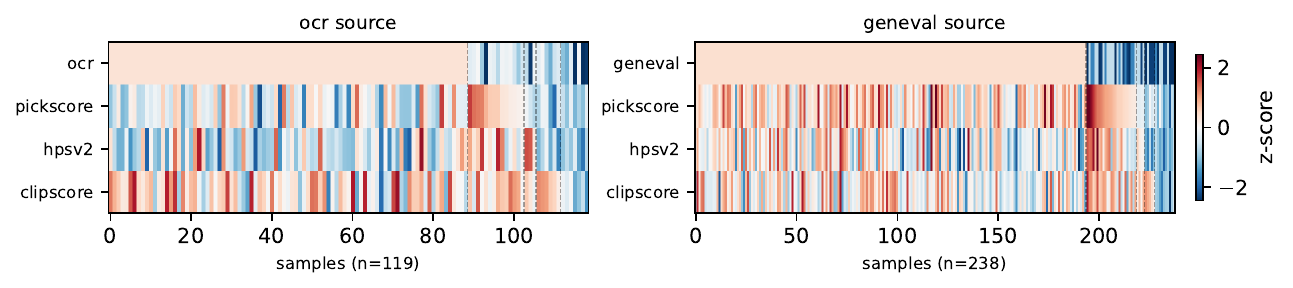}
    \vspace{-5mm}
    \caption{\textbf{Sample-level specialist structure.}
Each column denotes one sample, and each row shows its per-reward $z$-score advantage $A_k(x)$. High advantages are concentrated on source-specific rewards such as OCR or GenEval. Few samples achieve positive rewards across all dimensions.}
    \label{fig:specialist}
    \vspace{-5mm}
\end{figure}

\section{Related Work}
\label{sec:related}

\subsection{Reinforcement Learning for Diffusion Models}
\label{sec:related_rl}

Diffusion models~\citep{ho2020denoising, song2020score, song2020denoising} have become the dominant paradigm for high-fidelity image generation. Latent diffusion~\citep{rombach2022high,podell2023sdxl} moved the generation process into a compressed latent space, enabling efficient high-resolution synthesis, while subsequent scaling efforts~\citep{esser2024scaling} further improved generation quality by combining rectified flow formulations~\citep{liu2022flow,lipman2022flow} with transformer-based architectures~\citep{peebles2023scalable}. Diffusion models have since been extended far beyond text-to-image generation to a wide range of generative tasks, including image customization~\citep{zhang2023adding,tan2025ominicontrol,mou2024t2i,ye2023ip}, image editing~\citep{brooks2023instructpix2pix,batifol2025flux,wu2025qwenimagetechnicalreport,wang2026geometry}, video editing~\citep{jiang2025vace,zhao2025tinker}, image understanding~\citep{zhao2025diception,visionbanana2026} and even long-form video and movie generation~\citep{zhao2024moviedreamer,huang2025self,li2025stable,xiao2025captain}. 

Reinforcement Learning~\citep{schulman2017proximal,rafailov2023direct} has emerged as a primary approach for aligning models with human preferences. In diffusion RL, a reward model evaluates each generated sample, and the diffusion policy is optimized to maximize expected reward while remaining close to a pre-trained reference model~\citep{black2023training,fan2023dpok,tong2025noise}. Early work mainly relied on policy-gradient-based methods~\citep{black2023training,fan2023dpok}. More recently, inspired by the success of GRPO~\citep{shao2024deepseekmath} in large language models, a growing body of work has adapted similar ideas to diffusion models~\citep{liu2025flow,tong2026alleviating,xue2025dancegrpo,he2025tempflow,zhang2026op,li2025mixgrpo}, achieving stronger empirical performance. Recent work such as DiffusionNFT~\citep{zheng2026DiffusionNFT} has further improved training efficiency. Despite these advances, existing diffusion RL methods largely optimize a single scalar reward. When multiple reward signals are available, practitioners typically either train separate models for different rewards, fine-tune sequentially on different datasets, or combine several rewards through a weighted sum. None of these strategies provides a principled way to jointly optimize multiple quality dimensions within a single training run without manual reward weighting.

\subsection{Multi-Task Learning}
\label{sec:related_mtl}

Multi-task learning~\citep{deb2011multi,desideri2012multiple,sener2018multi,yu2020gradient,liu2021conflict,navon2022multi,liu2024stochastic} trains a shared model over multiple objectives and faces a closely related challenge: inter-task gradient interference can cause a single update to improve some objectives while harming others. To address this issue, prior work has developed a range of gradient-level optimization strategies, including finding the minimum-norm point in the convex hull of per-task gradients~\citep{desideri2012multiple,sener2018multi}, projecting out destructive gradient components~\citep{yu2020gradient}, maximizing worst-case per-task improvement~\citep{liu2021conflict}, and formulating gradient balancing as a game-theoretic bargaining problem~\citep{navon2022multi}. These methods share a common principle: resolving interactions among objectives in gradient space rather than loss space. They have proven effective in supervised multi-task settings, particularly for jointly learning multiple vision tasks.

RL~\citep{zhu2025active,zhong2025omni,shao2024deepseekmath} and Multi-reward alignment~\citep{zhou2024beyond,rame2023rewarded,shi2024decoding} has also received growing attention in Large Language Model. However, to the best of our knowledge, there has been little attempt to address the corresponding problem in diffusion RL. \ours{} bridges this gap by adapting gradient harmonization to the diffusion RL setting, with per-reward advantage decomposition and scale-aware gradient balancing tailored to the diffusion training objective.
\section{Method}
\label{sec:method}

\begin{figure}[t]
    \centering
    \includegraphics[width=\textwidth]{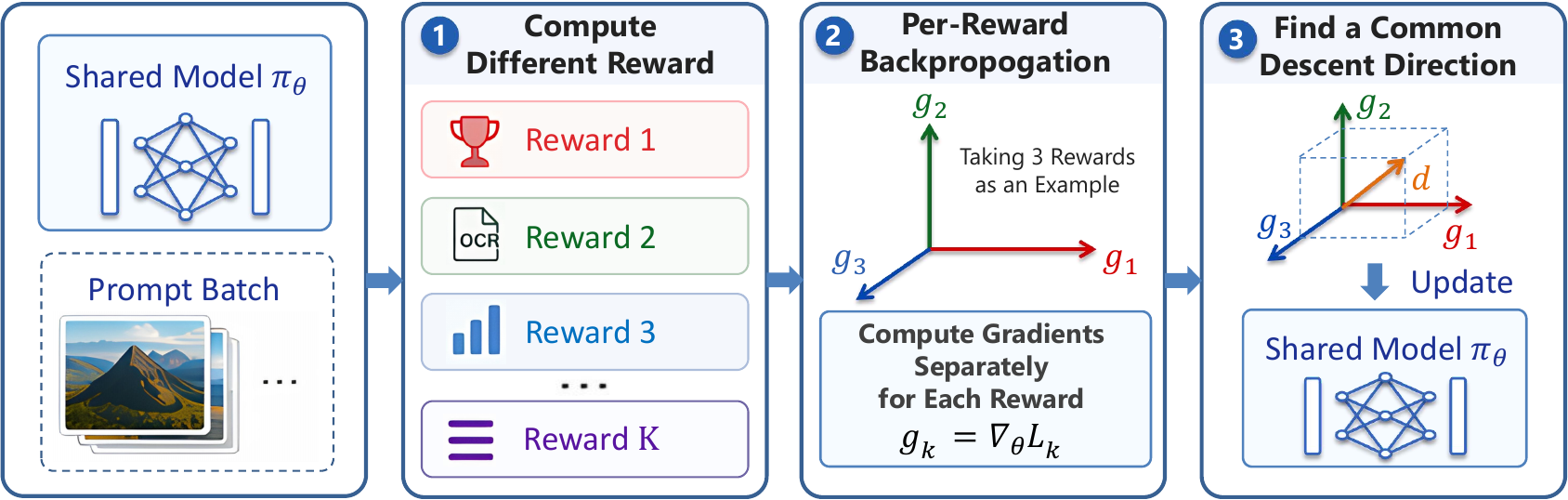}
    \caption{Overview of \ours{}. Given a prompt batch, the shared model $\pi_\theta$ generates images that are scored by $K$ reward models independently. Per-reward policy gradients $g_k = \nabla_\theta \mathcal{L}_k$ are computed via separate backpropagation passes. The gradient harmonization finds a common descent direction $d$ that balances all reward objectives and the shared model is updated accordingly.}
    \label{fig:method}
\end{figure}

\subsection{Preliminaries: DiffusionNFT}
% \citep{zheng2026DiffusionNFT}}
\label{sec:prelim}

Let $\pi_\theta$ denote a diffusion model parameterized by $\theta$, and let $\pi_{\mathrm{ref}}$ denote the frozen pre-trained reference policy. Given a single reward function $R: \mathcal{X} \to \mathbb{R}$, diffusion RL optimizes
\begin{equation}
    \max_\theta \; \mathbb{E}_{x \sim \pi_\theta}[R(x)] - \beta_{\mathrm{KL}} \cdot D_{\mathrm{KL}}(\pi_\theta \| \pi_{\mathrm{ref}}),
    \label{eq:single_reward}
\end{equation}
where $\beta_{\mathrm{KL}}$ controls the regularization strength. We build on DiffusionNFT \citep{zheng2026DiffusionNFT}, which implements Equation  \eqref{eq:single_reward} through a noise-free training (NFT) loss. For a generated sample $x$ with advantage $A(x)$, the NFT loss interpolates between a positive term that moves the model toward better predictions and a negative term that pushes it away:
\begin{equation}
    \ell(\theta; x, t) = r \cdot \mathcal{L}^+(\theta) + (1-r) \cdot \mathcal{L}^-(\theta),
    \label{eq:nft_loss}
\end{equation}
where
$r = \operatorname{clamp}\,\!\bigl(\tfrac{1}{2} + \tfrac{A(x)}{2A_{\max}},\; 0,\; 1\bigr)$
maps the advantage to an interpolation coefficient. Here $\mathcal{L}^+(\theta) = \|v_\theta^+ - v\|^2$ and $\mathcal{L}^-(\theta) = \|v_\theta^- - v\|^2$ are velocity prediction losses, with $v_\theta^+ = (1-\beta)v^{\mathrm{old}} + \beta v_\theta$ and $v_\theta^- = (1+\beta)v^{\mathrm{old}} - \beta v_\theta$ constructed from the current policy $v_\theta$ and the reference policy $v^{\mathrm{old}}$, and $v$ denotes the ground-truth velocity target. A key structural property is that $\mathcal{L}^+$ and $\mathcal{L}^-$ depend only on $\theta$ and the current sample, and are therefore \emph{independent of the advantage value}. The advantage affects the loss only through the affine mapping to $r$.

When multiple rewards $\{R_k\}_{k=1}^K$ are available, the standard approach first aggregates them into a scalar reward $R(x)=\sum_k w_k R_k(x)$, then derives a single advantage $A(x)$ and applies Equation~\eqref{eq:nft_loss} with one interpolation coefficient $r$. As we show next, this scalarization obscures which reward dimensions each sample is actually informative for, leading to poorly aligned updates in multi-reward training.

\subsection{Why Scalar Reward Aggregation Fails}
\label{sec:failure}

The Introduction identifies specialist samples as the sample-level reason that scalar reward aggregation is unreliable; at the gradient level, the weighted-sum update has negative worst-reward alignment in $80\%$ of the measured mini-batches, whereas \ours{} keeps the worst-reward alignment positive in all measured mini-batches (Appendix~\ref{app:harmony_diagnostics}). This gradient-level diagnostic motivates the harmonization procedure introduced next.

\subsection{Multi-Reward Gradient Harmonization}
\label{sec:harmonization}

\paragraph{Per-reward advantage decomposition.}
To preserve reward-specific supervision, \ours{} decomposes the training signal along reward dimensions. Following DiffusionNFT, for each reward $R_k$, we maintain an independent advantage estimator that normalizes $R_k$ within prompt groups:
\begin{equation}
    A_k(x) = \frac{R_k(x) - \mu_k(\mathrm{prompt})}{\sigma_k(\mathrm{prompt}) + \varepsilon},
    \label{eq:advantage}
\end{equation}
where $\mu_k$ and $\sigma_k$ are the running mean and standard deviation of $R_k$ for the same text prompt. Each $A_k$ yields a separate interpolation coefficient $r_k \in [0,1]$, which defines a reward-specific NFT loss $\ell_k$ through Equation \eqref{eq:nft_loss}. A backward pass through $\ell_k$ produces the corresponding policy gradient
\begin{equation}
    g_k = \nabla_\theta \; \frac{1}{NT} \sum_{i=1}^N \sum_{t=1}^T \ell_k(\theta; x_i, t).
    \label{eq:per_grad}
\end{equation}
All $K$ gradients are computed on the same sampled batch; only the advantage signal differs across rewards. This decomposition allows each sample to be credited precisely on the dimensions for which it is informative, instead of forcing all information through a single aggregated advantage.

\paragraph{Gradient normalization and harmonization.}
Different reward models can induce gradients at 
%very
drastically 
different scales. To remove this scale disparity from the harmonization step, \ours{} first normalizes each gradient:
\begin{equation}
    \hat{g}_k = g_k / \|g_k\|.
    \label{eq:normalize}
\end{equation}
Given the normalized gradients $\{\hat{g}_k\}_{k=1}^K$, \ours{} computes a unified update direction by solving a convex quadratic program,  as previously shown in multi-task learning \citep{desideri2012multiple,sener2018multi}:
\begin{equation}
    \alpha^* = \arg\min_{\alpha \in \Delta^K} \left\| \sum_{k=1}^K \alpha_k \hat{g}_k \right\|^2,
    \label{eq:mgda}
\end{equation}
where $\Delta^K = \{\alpha \in \mathbb{R}^K_{\ge 0} : \sum_k \alpha_k = 1\}$ is the probability simplex. 
The solution gives a descent direction that improves all rewards as shown in \cite{desideri2012multiple}.
The resulting direction
$
d^* = \sum_{k=1}^K \alpha_k^* \hat{g}_k
$
is the minimum-norm point in the convex hull of the normalized gradients and provides a balanced compromise across reward dimensions. When rewards are already aligned, the solution concentrates on their shared direction; when rewards emphasize different aspects, the solver adaptively reweights them according to the current batch. 
% In practice, the QP has only $K$ variables (typically $\le 8$) and is solved in under 1\,ms with SLSQP algorithm.

\paragraph{Rescaling and KL-decoupled update.}
Because $d^*$ is computed from unit-normalized gradients, its magnitude no longer matches the scale expected by the optimizer or the KL schedule. We therefore restore the natural update scale by multiplying $d^*$ by the mean norm of the original gradients:
\begin{equation}
    d_{\mathrm{final}} = d^* \cdot \bar{n}, \qquad
    \bar{n} = \frac{1}{K} \sum_{k=1}^K \|g_k\|.
    \label{eq:rescale}
\end{equation}
This normalize-then-rescale procedure separates directional balancing from step-size calibration.
The final parameter update combines the rescaled reward gradient with KL regularization as a separate term:
\begin{equation}
    \theta \leftarrow \theta - \eta \Bigl(d_{\mathrm{final}} + \beta_{\mathrm{KL}} \cdot \nabla_\theta D_{\mathrm{KL}}(\pi_\theta \| \pi_{\mathrm{ref}})\Bigr).
    \label{eq:update}
\end{equation}
We treat KL regularization outside the harmonization solve because it plays a different role from reward optimization: reward gradients determine \emph{which} aspects to improve, while the KL term controls \emph{how far} the policy is allowed to deviate from the reference model.

\subsection{Amortized Gradient Harmonization}
\label{sec:amortized}

The full harmonization procedure requires $K{+}1$ backward passes per iteration ($K$ reward-specific passes plus one KL pass), which becomes expensive as the number of rewards grows. Moreover, solving for $\alpha^*$ at every step introduces additional variance, since the harmonization weights are estimated from a single mini-batch and may fluctuate considerably across iterations. We observe that this instability can lead to undesirable visual artifacts at later training stages, even when average reward scores continue to improve. This motivates an amortized variant that reduces both computational overhead and short-term weight fluctuation.

\paragraph{Scalarization equivalence.}
Recall from Equation \eqref{eq:nft_loss} that the NFT loss depends on the advantage only through the affine mapping $r = A/(2A_{\max}) + 1/2$, while $\mathcal{L}^+(\theta)$ and $\mathcal{L}^-(\theta)$ are independent of $A$. This yields the following exact equivalence.

\begin{proposition}
\label{prop:equiv}
Let $\alpha \in \Delta^K$ and let $A_1,\ldots,A_K$ be per-reward advantages with $|A_k| < A_{\max}$ for all $k$ and $\bigl|\sum_k \alpha_k A_k\bigr| < A_{\max}$. Define the combined advantage $\bar{A} = \sum_{k=1}^K \alpha_k A_k$. Then
\begin{equation}
    \nabla_\theta \ell(\theta; \bar{A})
    =
    \sum_{k=1}^K \alpha_k \nabla_\theta \ell_k(\theta; A_k).
    \label{eq:equiv}
\end{equation}
\end{proposition}

\textit{Proof.}
Since each advantage $A_k$ is a fixed scalar independent of $\theta$,
\begin{equation}
\begin{split}
\nabla_\theta \ell_k
&= r_k \nabla_\theta \mathcal{L}^+ + (1-r_k)\nabla_\theta \mathcal{L}^- ,\\
\sum_k \alpha_k \nabla_\theta \ell_k
&= \left(\sum_k \alpha_k r_k\right)\nabla_\theta \mathcal{L}^+
 + \left(1-\sum_k \alpha_k r_k\right)\nabla_\theta \mathcal{L}^- ,
\end{split}
\end{equation}
where we used $\sum_k \alpha_k=1$. Because $r_k = \frac{A_k}{2A_{\max}} + \frac{1}{2}$,
\begin{equation}
\sum_k \alpha_k r_k
=
\sum_k \alpha_k \Bigl(\frac{A_k}{2A_{\max}} + \frac{1}{2}\Bigr)
=
\frac{\bar{A}}{2A_{\max}} + \frac{1}{2}
=
\bar{r}.
\end{equation}
Substituting $\bar{r}$ recovers $\nabla_\theta \ell(\theta;\bar{A})$. \textbf{This shows that, when the clamp is inactive, the convex combination of the per-reward NFT gradients can be recovered exactly by a single backward pass using the combined advantage $\bar{A}$.}
The equivalence relies on two properties: (i) the NFT loss depends on the advantage only through the \emph{affine} map $r = A/(2A_{\max}) + 1/2$, and (ii) the simplex constraint $\sum_k \alpha_k = 1$ preserves the constant offset under convex combination. The clamp in Equation~\ref{eq:nft_loss} introduces a bounded deviation only when $|A_k| \ge A_{\max}$. Following DiffusionNFT~\cite{zheng2026DiffusionNFT}, we set $A_{\max}=5$ during training, which serves as a loose safety bound. Empirically, we never observed the clamp being activated during training.

\paragraph{Amortized procedure.}
Proposition~\ref{prop:equiv} enables an efficient application of fixed reward-balancing coefficients through a single NFT backward pass. We emphasize that gradient normalization is used only when estimating the coefficients $\alpha^*$: solving Equation~\eqref{eq:mgda} with normalized gradients removes reward-dependent scale disparities and makes $\alpha^*$ reflect the directional conflict among reward objectives. In contrast, the amortized update applies the cached coefficients in the advantage space, because the exact single-backward equivalence holds for convex combinations of the NFT losses, or equivalently of the unnormalized per-reward gradients. Recovering a normalized-gradient combination at every amortized step would require per-reward gradient norms and thus defeat the purpose of amortization.

Therefore, every $N$ steps, we run the full harmonization procedure to refresh $\alpha^*$ from normalized gradient. During the intervening $N{-}1$ steps, we form $\bar{A}=\sum_k \alpha_k^* A_k$ using the cached coefficients and perform only one reward backward pass. This coefficient-amortized approximation retains the scale-invariant reward-balancing information estimated by full harmonization, while preserving the natural gradient scale of the current NFT loss and reducing the average per-step cost from $(K{+}1)\times$ to $(K+N)/N$ times that of a single-reward baseline.

\subsection{Coefficient Smoothing for Stable Amortization}
\label{sec:coef_smoothing}

While amortized harmonization reduces the computational cost of training, it also makes the optimization more sensitive to short-term fluctuations in the estimated balancing coefficients. In particular, we observe that some reward dimensions may receive little or no useful signal from a rollout batch, especially during the early stages of training. This often happens for specialist rewards that require precise compositional or spatial correctness: when none of the generated samples satisfies the corresponding constraint, the estimated gradient can become uninformative, and the harmonization solver may assign a near-zero coefficient to that reward. Under amortization, such a transient zero coefficient is then reused for the following $N{-}1$ steps, effectively suppressing that reward throughout the entire amortization window. This can slow down training and reduce final performance.

To improve the stability of amortized harmonization, we apply exponential moving average (EMA) smoothing to the balancing coefficients. Let $\alpha_t^*$ denote the coefficients obtained from the full harmonization step at iteration $t$. Instead of directly using $\alpha_t^*$ for the subsequent amortized updates, we maintain a smoothed coefficient vector 
$\bar{\alpha}_t$:
\begin{equation}
% $
    \bar{\alpha}_t = \rho \bar{\alpha}_{t-1} + (1-\rho)\alpha_t^*,
    \label{eq:coef_ema}
% $
\end{equation}
where $\rho$ is the EMA decay. Since both $\bar{\alpha}_{t-1}$ and $\alpha_t^*$ lie on the probability simplex, their convex combination also remains a valid simplex vector. We then use $\bar{\alpha}_t$ to construct the combined advantage during amortized updates:
$
    \bar{A} = \sum_{k=1}^K \bar{\alpha}_{t,k} A_k .
$
This smoothing mechanism prevents occasional rollout failures from completely removing a reward signal over an amortization window, while still allowing the coefficients to adapt to the gradient geometry estimated by the full harmonization step. In all experiments, we set the EMA decay to $\rho=0.7$. Empirically, coefficient smoothing improves both training efficiency and effectiveness.

\section{Experiments}
\label{sec:exp}

\subsection{Experimental Setup}
\label{sec:setup}

We build on Stable Diffusion 3.5 Medium~\citep{esser2024scaling} and fine-tune LoRA adapters~\citep{Hu2021LoRALA} with rank 32 and alpha 64 using the NFT loss in Equation~\ref{eq:nft_loss}. Unless otherwise specified, we use AdamW with a constant learning rate of $3 \times 10^{-4}$. Our training objective jointly optimizes five rewards: three general-purpose rewards, PickScore~\citep{kirstain2023pick}, HPSv2~\citep{wu2023human}, and CLIPScore~\citep{hessel2021clipscore}, and two specialist rewards, OCR accuracy and GenEval~\citep{ghosh2023geneval}. To assess transfer beyond the optimized rewards, we additionally report Aesthetic Score~\cite{schuhmann2022laion}, ImageReward~\citep{xu2023imagereward}, and UniReward~\citep{wang2025unified}, none of which are used during training. The model is trained with 16 NVIDIA H200 GPUs.
We compare against single-reward FlowGRPO specialists~\citep{liu2025flow}, each optimized for one reward, and two multi-reward DiffusionNFT variants~\citep{zheng2026DiffusionNFT}: sequential$^\dagger$, which follows a manually scheduled multi-stage training procedure, and simultaneous$^\ddagger$, which directly scalarizes all rewards. All methods are evaluated under the same framework for fair comparison.

\subsection{Main Results}
\label{sec:main_results}

\begin{table*}[t]
    \centering
    \renewcommand{\arraystretch}{1.1}
    \vspace{-2mm}
    \caption{\textbf{Main results.} Comparison of \ours{} with pre-trained diffusion models and RL fine-tuning methods. \ours{} jointly optimizes all in-domain rewards in a single run. $^\dagger$Sequential multi-stage training with hand-crafted curriculum. $^\ddagger$Simultaneous five-reward training with weighted-sum aggregation. \colorbox{llgray}{Gray} denotes in-domain reward used during training. \textbf{Bold} denotes best and \underline{underline} denotes second best. Composite is the per-row mean of column-wise z-scores (each metric standardized to zero mean and unit variance across the rows of this table); higher is better. Evaluated with the DiffusionNFT official code.}
    \vspace{1mm}
    \resizebox{\linewidth}{!}{
        \begin{tabular}{lccccccccc}
            \toprule
            \multirow{2}{*}{\textbf{Model}} 
            & \multicolumn{2}{c}{\textbf{Rule-Based}} 
            & \multicolumn{6}{c}{\textbf{Model-Based}} 
            & \multirow{2}{*}{\textbf{Composite} $\uparrow$} \\
            \cmidrule(lr){2-3} \cmidrule(lr){4-9}
            & \textbf{GenEval} & \textbf{OCR}  
            & \textbf{PickScore} & \textbf{CLIPScore} & \textbf{HPSv2.1} 
            & \textbf{Aesthetic} & \textbf{ImgRwd} & \textbf{UniRwd} 
            & \\
            \midrule
            SD-XL & 0.55 & 0.14 & 22.42 & 0.287 & 0.280 & 5.60 & 0.76 & 2.93 & -0.455 \\
            SD3.5-L & 0.71 & 0.68 & 22.91 & 0.289 & 0.288 & 5.50 & 0.96 & 3.25 & +0.116 \\
            FLUX.1-Dev & 0.66 & 0.59 & 22.84 & \textbf{0.295} & 0.274 & 5.71 & 0.96 & 3.27 & +0.104 \\
            \midrule
            SD3.5-M (w/o CFG) & 0.24 & 0.12 & 20.51 & 0.237 & 0.204 & 5.13 & -0.58 & 2.02 & -2.319 \\
            + CFG & 0.63 & 0.59 & 22.34 & 0.285 & 0.279 & 5.36 & 0.85 & 3.03 & -0.255 \\
            \quad+ FlowGRPO
            & \cellcolor{llgray}\textbf{0.95} & 0.66 & 22.51 & \underline{0.293} & 0.274 & 5.32 & 1.06 & 3.18 & +0.120 \\
            & 0.66 & \cellcolor{llgray}\underline{0.92} & 22.41 & 0.290 & 0.280 & 5.32 & 0.95 & 3.15 & +0.013 \\
            & 0.54 & 0.68 & \cellcolor{llgray}\underline{23.50} & 0.280 & 0.316 & 5.90 & 1.29 & 3.37 & +0.362 \\
            + DiffusionNFT$^\dagger$
            & \cellcolor{llgray}\underline{0.94} & \cellcolor{llgray}0.91 & \cellcolor{llgray}\textbf{23.80} & \cellcolor{llgray}\underline{0.293} & \cellcolor{llgray}\underline{0.331} & 6.01 & \underline{1.49} & \underline{3.49} & \underline{+1.015} \\
            + DiffusionNFT$^\ddagger$
            & \cellcolor{llgray}0.92 & \cellcolor{llgray}0.91 & \cellcolor{llgray}21.53 & \cellcolor{llgray}0.267 & \cellcolor{llgray}0.300 & \underline{6.15} & 1.16 & 3.04 & +0.184 \\
            + \oursbf{}
            & \cellcolor{llgray}\underline{0.94} & \cellcolor{llgray}\textbf{0.96} & \cellcolor{llgray}22.83 & \cellcolor{llgray}0.286 & \cellcolor{llgray}\textbf{0.355} & \textbf{6.59} & \textbf{1.53} & \textbf{3.52} & \textbf{+1.116} \\
            \bottomrule
        \end{tabular}
    }
    \vspace{-1mm}
    \label{tab:main}
\end{table*}
\paragraph{Performance.}

\begin{figure}[ht]
    \centering
    \includegraphics[width=\linewidth]{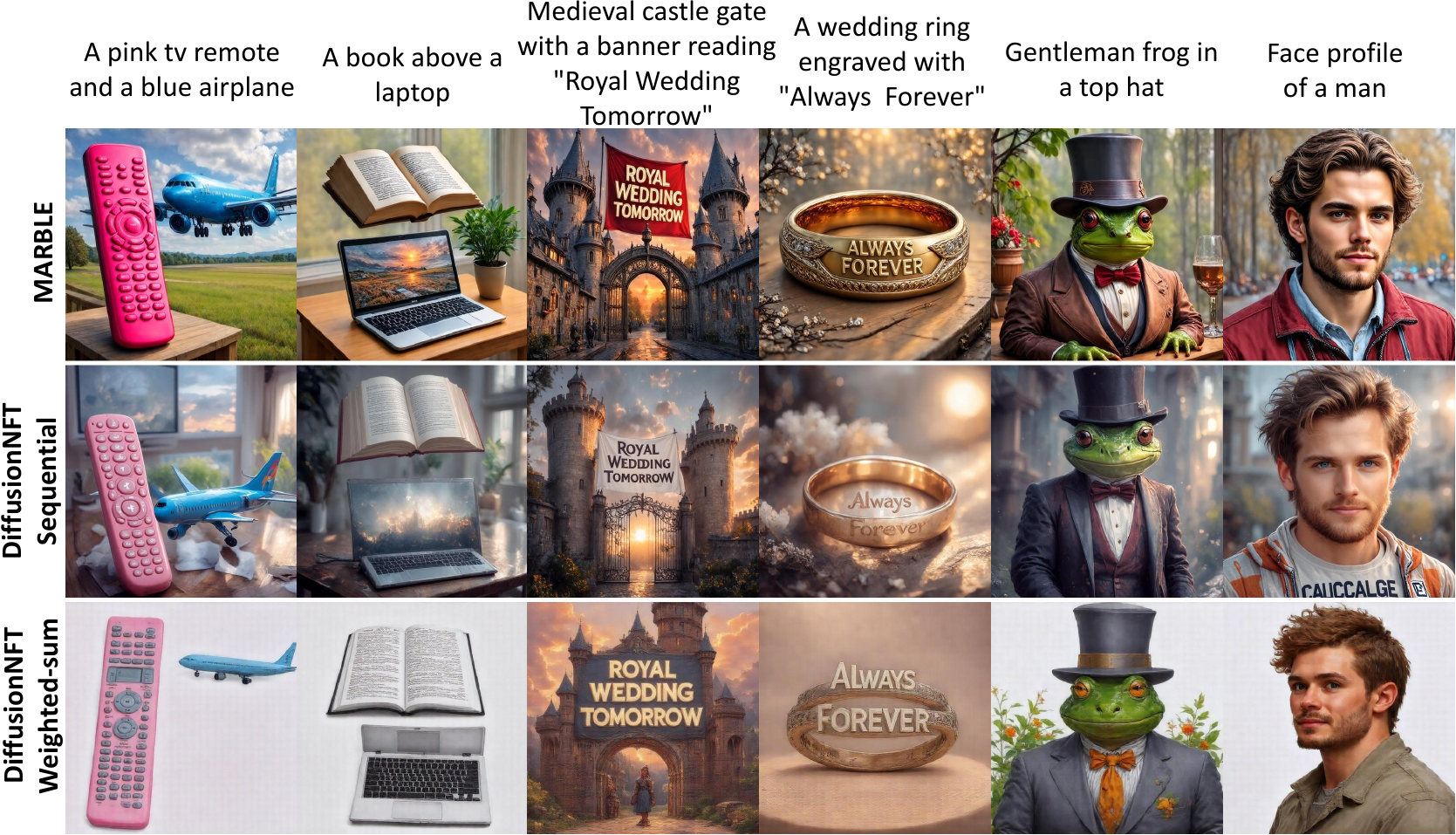}
    \caption{Visualizations of qualitative comparisons between \ours{} and Baselines.}
    \label{fig:main_results}
\end{figure}

Table~\ref{tab:main} reports the main quantitative results. Single-reward FlowGRPO specialists excel on their target objective but transfer poorly to others, requiring a separate model per reward and offering limited cross-objective coverage. In contrast, \ours{} improves all five training rewards within a single model. Qualitative comparisons in Figure~\ref{fig:main_results} further show that \ours{} simultaneously satisfies diverse reward dimensions, while the weighted-sum baseline fails to do so.

Among multi-reward baselines, DiffusionNFT$^\dagger$ (sequential) reaches comparable quality, but at considerable practical cost: it requires a manually scheduled multi-stage curriculum where practitioners must decide the ordering of rewards, the number of steps per stage, and the dataset assigned to each stage, all of which are sensitive to a lot of hyperparameters. More importantly, after introducing a new reward, sequential training \textbf{often needs to revisit previously seen rewards to mitigate forgetting}, so the schedule grows with the number of rewards rather than scaling automatically. The alternative DiffusionNFT$^\ddagger$ (simultaneous) avoids this manual scheduling by directly scalarizing all rewards into a single objective, but as a result performs substantially worse on the specialist objectives. \ours{} matches or exceeds both baselines from a single joint training run, with no per-stage hyperparameters and no explicit replay schedule.

DiffusionNFT$^\dagger$ does score moderately higher on PickScore and slightly higher on CLIPScore, but \ours{} matches or surpasses it on every other reward and ranks first on the four held-out quality metrics (HPSv2.1, Aesthetic, ImageReward, UniReward). To summarize across all eight metrics, we report a Composite score in the last column of Table~\ref{tab:main}, computed as the mean of column-wise z-scores so each metric contributes equally regardless of scale. \ours{} attains the highest Composite, showing that the small concession on PickScore/CLIPScore is more than offset by gains elsewhere. Appendix~\ref{app:user_study} further demonstrates that \ours{}'s results are preferred.

\begin{table}[t]
\vspace{-1mm}
\caption{Training efficiency comparison on 8$\times$H200. We report relative training speed and GPU memory, both normalized by the weighted-sum baseline.}
\vspace{-1mm}
\label{tab:cost}
\centering
\begin{tabularx}{\textwidth}{
    >{\raggedright\arraybackslash}p{0.56\textwidth}
    >{\centering\arraybackslash}p{0.20\textwidth}
    >{\centering\arraybackslash}p{0.20\textwidth}
}
\toprule
Method & Relative speed & GPU memory \\
\midrule
Weighted Sum ($K{=}5$, DiffusionNFT Baseline) & $1.00\times$ & $59$G ($1.00\times$) \\
\ours{} w/ amortization ($K{=}5$, $N{=}10$) & $0.97\times$ & $67$G ($1.14\times$) \\
\ours{} w/o amortization ($K{=}5$) & $0.56\times$ & $67$G ($1.14\times$) \\
\bottomrule
\end{tabularx}
\vspace{-4mm}
\end{table}

\paragraph{Training efficiency.}
We further show the training efficiency comparison in Table~\ref{tab:cost}. All results are measured on 8$\times$H200 and normalized by the weighted-sum baseline. Full per-reward harmonization introduces noticeable overhead, reducing the relative training speed to $0.56\times$ due to the need for multiple reward-specific backward passes. In contrast, the amortized variant substantially reduces this overhead and achieves $0.97\times$ relative speed, which is close to the weighted-sum baseline. Both \ours{} variants require only a modest increase in GPU memory, from $59$G to $67$G per GPU, corresponding to $1.14\times$ relative memory. These results show that amortization makes gradient-space reward balancing practical at nearly the same training speed as scalarized multi-reward training.

\subsection{Ablation Studies and Analysis}
\label{sec:ablation}

\begin{table*}[t]
    \centering
    \renewcommand{\arraystretch}{1.1}
    \vspace{-2mm}
    \caption{\textbf{Ablation study.} Each row removes or replaces one component of \ours{}. All variants use the same 5-reward setup and training budget.}
    \resizebox{\linewidth}{!}{
        \begin{tabular}{lcccccccc}
            \toprule
            \multirow{2}{*}{\textbf{Variant}} 
            & \multicolumn{2}{c}{\textbf{Rule-Based}} 
            & \multicolumn{6}{c}{\textbf{Model-Based}} \\
            \cmidrule(lr){2-3} \cmidrule(lr){4-9}
            & \textbf{GenEval} 
            & \textbf{OCR} 
            & \textbf{PickScore} 
            & \textbf{CLIPScore} 
            & \textbf{HPSv2.1} 
            & \textbf{Aesthetic} 
            & \textbf{ImgRwd} 
            & \textbf{UniRwd} \\
            \midrule
            \ours{} (full, $\rho = 0.7$ ) 
            & 0.93 & 0.96 & 22.62 & 0.283 & 0.355 & 6.59 & 1.52 & 3.45 \\
            \midrule

            % \multicolumn{9}{l}{\textit{Main design ablations}} \\
            w/o gradient normalization 
            & \multicolumn{8}{c}{\textbf{FAIL}} \\
            Fixed $\alpha=0.2$
            & 0.86 & 0.89 & 22.64 & 0.272 & 0.346 & 6.55 & 1.45 & 3.42 \\
            Solve $\alpha$ every step 
            & 0.92 & 0.92 & 21.32 & 0.267 & 0.301 & 5.89 & 1.17 & 3.04 \\
            \bottomrule
        \end{tabular}
    }
    \vspace{-3mm}
    \label{tab:ablation}
\end{table*}

Table~\ref{tab:ablation} ablates the main design choices that determine the update direction in \ours{}: replacing adaptive coefficients with fixed uniform coefficients ($\alpha_k=0.2$), removing gradient normalization before solving $\alpha$, and solving $\alpha$ at every step instead of using the amortized update. Due to space constraints, we report the headline ablations in the main text and defer the supporting analyses to Appendix~\ref{app:additional_ablations}, including training dynamics and coefficient adaptation (Appendix~\ref{app:training_curves}), amortization-interval sensitivity (Appendix~\ref{app:amortization_interval}), EMA-decay sensitivity (Appendix~\ref{app:ema_decay}), and alternative heuristic balancing strategies (Appendix~\ref{app:alternative_strategies}).

\paragraph{Gradient normalization before solving $\alpha$.}
The harmonization coefficients $\alpha$ are computed from the per-reward gradients by solving the QP in Equation~\ref{eq:mgda}. Without gradient normalization, this optimization becomes highly sensitive to the raw magnitudes of different reward gradients. The resulting coefficients tend to be dominated by scale differences rather than the directional relationships among rewards. In practice, we observe that this often produces degenerate or numerically unstable coefficients, leading to failed optimization.

\paragraph{Equal $\alpha$ weighting.}
We examine a simple variant that uses fixed uniform coefficients, i.e., $\alpha_k=0.2$ for all five rewards. This setting leads to imbalanced convergence across different rewards. We observe that general-purpose rewards related to overall visual aesthetics improve relatively quickly, whereas more challenging specialist objectives, such as object attributes and spatial relations, remain under-optimized. In contrast, \ours{} dynamically adjusts the coefficients during training, allocating more optimization emphasis to tasks that are currently harder to improve. This adaptive allocation enables more balanced convergence across both general and specialist rewards.
We further find that the best final performance is obtained by using dynamic coefficients during most of training, followed by a short uniform-coefficient stage near the end. Intuitively, the dynamic stage helps the model allocate capacity to difficult reward dimensions, while the final uniform stage encourages all rewards to be jointly consolidated under an equal weighting. This strategy achieves the strongest overall balance.

\paragraph{Coefficient amortization.}
We also evaluate a variant that solves for $\alpha$ at every training step without amortization. Although this provides a fresh estimate of the gradient geometry at each iteration, it substantially increases training cost due to repeated per-reward backward passes. Moreover, the coefficients estimated from a single mini-batch can fluctuate considerably across iterations, injecting high-frequency variation into the update direction and negatively affecting training stability.

\section{Conclusion}
\label{sec:conclusion}

We propose \oursbf{}, the first multi-reward balancing method for diffusion model RL fine-tuning. \ours{} preserves reward-specific supervision through per-reward advantage decomposition and gradient harmonization, avoiding the specialist-sample dilution that limits weighted-sum aggregation, and its amortized formulation keeps training cost close to the baseline. One limitation of our current study is that we validate \ours{} primarily on image generation. Extending the framework to video diffusion and generative world models remains an important direction, as these settings involve richer and more heterogeneous quality dimensions, such as temporal consistency, motion realism, and physical plausibility, making reward balancing even more critical. Another promising direction is scaling \ours{} to larger reward sets, where both optimization and efficiency become tighter challenges. We believe that \ours{} provides an important step toward scalable multi-objective alignment for future generative models.
\clearpage

\clearpage
{\small
\bibliographystyle{plain}
\bibliography{main}
}

\clearpage

\appendix
\renewcommand\thesection{\Alph{section}}
\renewcommand\thefigure{S\arabic{figure}}
\renewcommand\thetable{S\arabic{table}}
\renewcommand\theequation{S\arabic{equation}}
\renewcommand\theHfigure{S\arabic{figure}}
\renewcommand\theHtable{S\arabic{table}}
\renewcommand\theHequation{S\arabic{equation}}
\setcounter{figure}{0}
\setcounter{table}{0}
\setcounter{equation}{0}

\section{Appendix Overview}
\label{app:overview}

This appendix provides supporting material for \ours{}, including qualitative examples, paper-level takeaways, extended ablations, implementation details, and future directions. The contents are organized as follows:
\begin{itemize}
    \item \textbf{Appendix~\ref{app:qualitative}: Additional Qualitative Comparisons.} 
    We provide additional qualitative examples illustrating how a single \ours{} model handles text rendering, attribute and spatial understanding, and counting while maintaining coherent visual quality.

    \item \textbf{Appendix~\ref{app:additional_ablations}: Additional Ablations and Analyses.} 
    We collect the main takeaways and extended empirical analyses behind \ours{}:
    \begin{itemize}
        \item \textbf{Appendix~\ref{app:key_takeaways}: Key Insights and Takeaways.} 
        A paper-level summary of the main MARBLE insights, including why scalar reward aggregation fails, how to interpret the learned coefficients, and which practical defaults are important when using \ours{}.

        \item \textbf{Appendix~\ref{app:harmony_diagnostics}: Update-Direction Harmony Diagnostics.}
        We provide the per-batch harmony visualization and aggregate statistics comparing weighted-sum and harmonized update directions.

        \item \textbf{Appendix~\ref{app:training_curves}: Training Dynamics and Coefficient Adaptation.} 
        We provide training curves, coefficient trajectories, and a closer analysis of how the learned coefficients relate to optimization difficulty.

        \item \textbf{Appendix~\ref{app:amortization_interval}: Amortization Interval.} 
        We analyze different coefficient refresh intervals and explain why $N{=}10$ is used as the main setting.

        \item \textbf{Appendix~\ref{app:ema_decay}: EMA Decay for Coefficient Smoothing.}
        We discuss how the EMA decay controls the stability-adaptivity trade-off in coefficient smoothing.

        \item \textbf{Appendix~\ref{app:alternative_strategies}: Alternative Heuristic Strategies.} 
        We compare \ours{} with heuristic reward-balancing strategies, including fixed uniform coefficients, reward grouping, and specialist reward up-weighting.

        \item \textbf{Appendix~\ref{app:user_study}: Human Preference Evaluation and Metric Correlations.}
We provide a human preference study and metric-correlation analysis to examine the benefit of improving multiple reward dimensions simultaneously. We show the importance of broad reward coverage: improving across multiple dimensions leads to more generally preferred outputs than optimizing a single metric in isolation.
    \end{itemize}

    \item \textbf{Appendix~\ref{app:implementation}: Additional Implementation Details.} 
    We describe implementation details for reproducing \ours{} in distributed training, focusing on how to extract, synchronize, and harmonize per-reward gradients under DDP.

    \item \textbf{Appendix~\ref{app:future_work}: Future Work.} 
    We discuss future directions, including scaling \ours{} to larger reward sets and extending reward-balanced optimization to video generation and generative world models.
\end{itemize}

\section{Additional Qualitative Comparisons}
\label{app:qualitative}

\begin{figure}[ht]
    \centering
    \includegraphics[width=\textwidth]{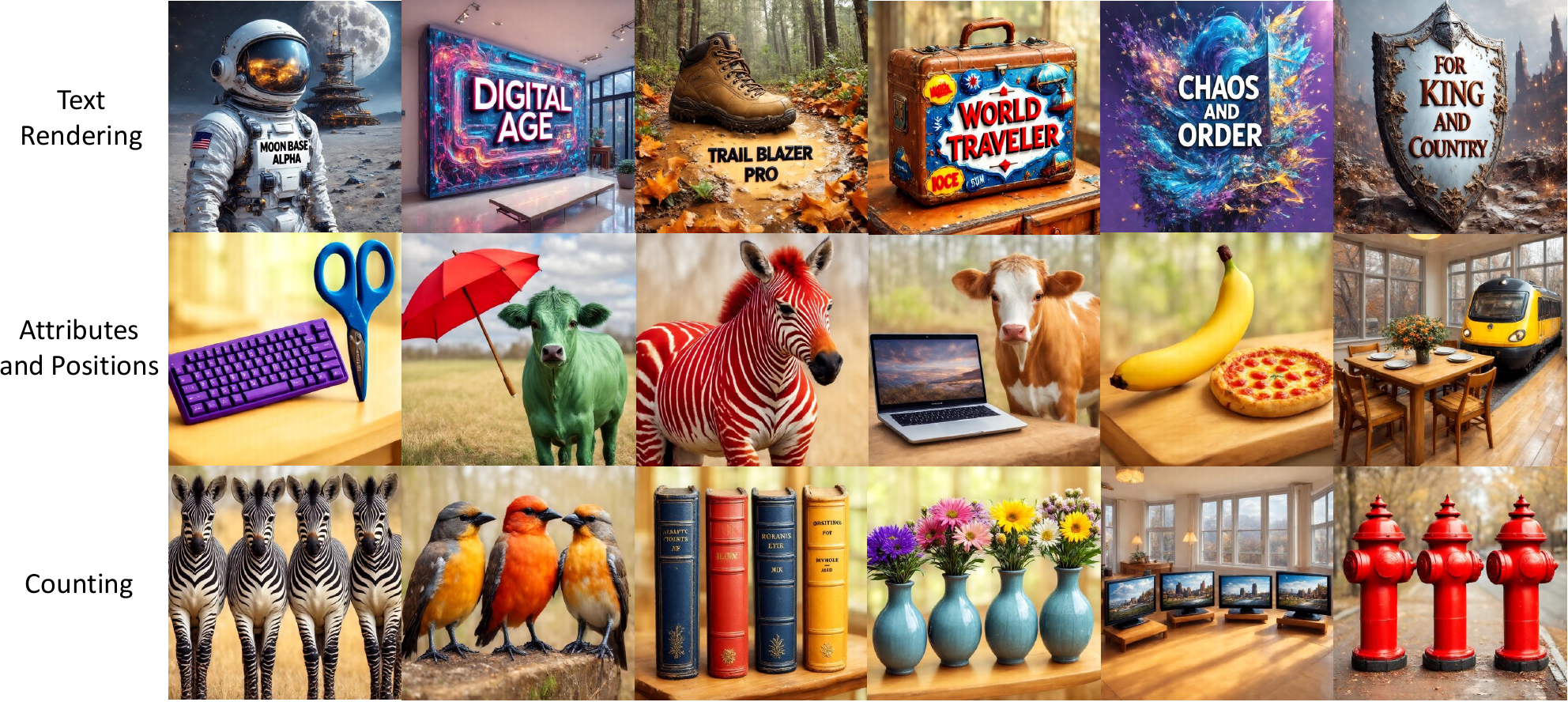}
    \caption{
    Additional qualitative results of \ours{}.
    Our one single model demonstrates simultaneous improvements in text rendering, attribute and position understanding, and counting.
    \ours{} generates legible text, preserves fine-grained attribute-object bindings and spatial layouts, and follows counting constraints while maintaining coherent visual quality.
    }
    \label{fig:additional_qualitative}
\end{figure}

We provide additional qualitative results in Figure~\ref{fig:additional_qualitative} and additional comparisons in Figure~\ref{fig:supp1},~\ref{fig:supp2},~\ref{fig:supp3} and~\ref{fig:supp4} to further illustrate the effectiveness of \ours{}. The visualizations cover three representative specialist capabilities: text rendering, attribute and spatial composition, and counting. As shown in Figure~\ref{fig:additional_qualitative}, \ours{} produces legible and semantically consistent text across diverse contexts, preserves fine-grained attribute-object bindings and spatial layouts, and follows counting constraints while maintaining coherent visual quality.

The comparisons in Figure~\ref{fig:supp1},~\ref{fig:supp2},~\ref{fig:supp3} and~\ref{fig:supp4} further highlight the limitations of existing baselines. The weighted-sum baseline often fails to improve all aspects simultaneously: although it can sometimes satisfy specific requirements such as object counts or attributes, its overall visual quality is visibly degraded. Compared with DiffusionNFT, \ours{} generates sharper images with fewer blur and distortion artifacts. Moreover, DiffusionNFT relies on a heavily tuned training schedule to obtain a competitive unified model, whereas \ours{} achieves balanced improvements through gradient-space reward harmonization. Overall, \ours{} better balances different reward dimensions, producing visually sharper images while more reliably satisfying text rendering, attribute binding, spatial layout, and counting requirements. These qualitative results are consistent with the quantitative improvements and demonstrate the strengths of \ours.

\section{Additional Ablations and Analyses}
\label{app:additional_ablations}

\subsection{Key Insights and Takeaways}
\label{app:key_takeaways}
This section summarizes the main takeaways of \ours{} as a whole: what problem it addresses, why its gradient-space design is effective, and how to use it in practice.
\begin{itemize}
    \item \textbf{The central problem is scalar reward aggregation, not merely a poor choice of scalar weights.} Table~\ref{tab:main} shows that direct simultaneous optimization with a weighted-sum reward underperforms on several reward dimensions. The gradient-alignment analysis in Figure~\ref{fig:harmony} provides further evidence that scalar aggregation weakens reward-specific optimization signals. This suggests that multi-reward diffusion RL should preserve reward-specific supervision rather than heuristically collapse all feedback into a single scalar reward.

    \item \textbf{The key contribution of \ours{} is gradient-space balancing with per-reward credit assignment.} By maintaining separate advantages and harmonizing reward-specific gradients, \ours{} outperforms the simultaneous weighted-sum baseline on all five optimized rewards within a single model, while matching or slightly surpassing the sequentially trained baseline that requires extensive manual schedule tuning (Table~\ref{tab:main}). The ablations in Table~\ref{tab:ablation} further show that both gradient normalization and adaptive harmonization are important: removing normalization leads to optimization failure, while fixed uniform coefficients, i.e., $\alpha_k=0.2$, weaken the balance between general image-quality rewards and harder specialist rewards.

    \item \textbf{$\alpha$ partially balances optimization across tasks with different difficulty levels.} 
    As shown in Figure~\ref{fig:alpha}, the smoothed coefficients do not simply track the corresponding raw reward curves. Instead, they appear to reflect, to some extent, the relative optimization difficulty of each reward. Easier image-quality rewards, such as HPSv2, tend to receive coefficients below the uniform baseline of $0.2$, whereas harder specialist rewards, such as GenEval, can receive larger coefficients, around $0.3$ during training. Therefore, a larger coefficient should not be interpreted as indicating that the reward value is higher; rather, it suggests that the current gradient geometry allocates more optimization emphasis to that reward.

    \item \textbf{Amortization and EMA smoothing are practical defaults, not just efficiency tricks.} Full per-step harmonization is conceptually clean but expensive and sensitive to batch-level coefficient fluctuations. The amortized update keeps the training speed close to the weighted-sum baseline (Table~\ref{tab:cost}), while EMA smoothing reduces abrupt changes in $\alpha$ between full harmonization steps and prevents a transient weak batch from suppressing a reward for an entire amortization window. We therefore use $N=10$ and $\rho=0.7$ as the default setting in the main experiments: $N=10$ provides similar performance to $N=5$ but is slightly faster because it refreshes coefficients less often, while avoiding the degradation observed at $N=20$ (Table~\ref{tab:amortization_interval}); $\rho=0.7$ is the default EMA setting reported in Table~\ref{tab:ablation_ema}.

    \item \textbf{Multi-dimensional image-quality improvement matters.}
    Image quality is multi-dimensional, which cannot be fully captured by any single reward model. Different metrics, such as PickScore, CLIPScore, HPSv2, Aesthetic Score, ImageReward, and UniReward, emphasize different aspects of generation quality, including human preference, text-image alignment, aesthetics, and perceptual realism. We observe that improving one metric alone does not necessarily translate to consistent gains on others, and can 
    sometimes lead to weaker overall visual quality. In contrast, \ours{} consistently achieves broader improvements across multiple image-quality metrics, suggesting that multi-reward balancing leads to more general and robust perceptual enhancement. Please see more details in Appendix~\ref{app:user_study}.
\end{itemize}

\subsection{Update-Direction Harmony Diagnostics}
\label{app:harmony_diagnostics}

\begin{figure}[ht]
    \centering
    \includegraphics[width=\textwidth]{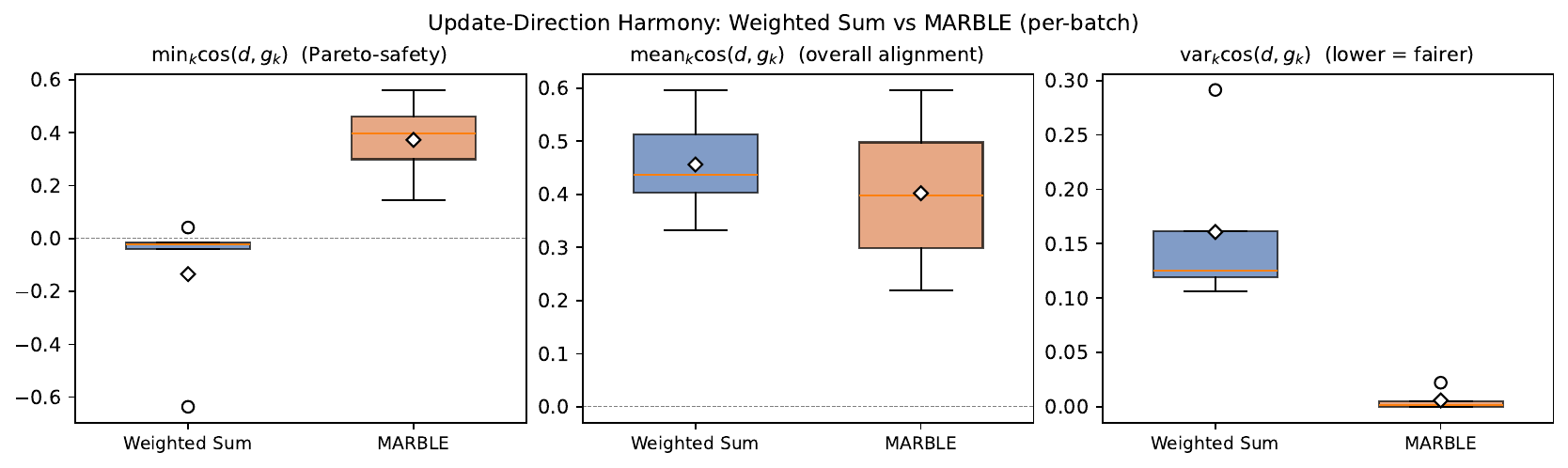}
    \caption{\textbf{Update-direction harmony.}
    Per-batch $\min_k$, $\mathrm{mean}_k$, and $\mathrm{var}_k$ of $\{\cos(d,g_k)\}$ for the weighted-sum direction $d=g_{\mathrm{ws}}=\tfrac{1}{K}\sum_k g_k$ and the harmonized direction $d=g_{\mathrm{marble}}$ (Section~\ref{sec:harmonization}). The corresponding aggregate statistics are reported in Table~\ref{tab:harmony_stats}.}
    \label{fig:harmony}
\end{figure}

\begin{table}[ht]
\caption{\textbf{Update-direction harmony statistics.} Averages over $n=5$ mini-batches. $\uparrow$ means larger is better; $\downarrow$ means smaller is better.}
\label{tab:harmony_stats}
\centering
\small
\setlength{\tabcolsep}{3pt}
\begin{tabularx}{\textwidth}{>{\raggedright\arraybackslash}Xccc}
\toprule
\textbf{Statistic} & \textbf{Weighted sum} & \textbf{\ours{}} & \textbf{$\Delta$} \\
\midrule
Worst-reward alignment $\min_k \cos(d,g_k)$ $\uparrow$ & $-0.1346$ & $+0.3721$ & $+0.5067$ \\
Average alignment $\mathrm{mean}_k \cos(d,g_k)$ $\uparrow$ & $+0.4559$ & $+0.4014$ & $-0.0545$ \\
Alignment imbalance $\mathrm{var}_k \cos(d,g_k)$ $\downarrow$ & $+0.1605$ & $+0.0058$ & $-0.1548$ \\
Conflict rate $P(\min_k \cos(d,g_k)<0)$ $\downarrow$ & $0.800$ & $0.000$ & $-0.800$ \\
\bottomrule
\end{tabularx}
\end{table}

For a mini-batch, let $g_k$ denote the policy gradient induced by reward $R_k$, and let $d$ denote the update direction produced by a multi-reward training rule. If $\cos(d,g_k)<0$ for some reward $k$, then the shared update is anti-aligned with that reward's own gradient on the same batch. Table~\ref{tab:harmony_stats} shows that the weighted-sum direction has a negative worst-reward cosine on average and produces a negative worst-reward alignment in $80\%$ of the measured mini-batches. In contrast, the harmonized direction raises the worst-reward cosine from $-0.1346$ to $+0.3721$ and eliminates negative-minimum batches in this measurement, while keeping the mean cosine similar. Its across-reward variance is also much smaller ($0.0058$ vs. $0.1605$), indicating that the update direction is more evenly aligned with the five reward gradients.

\subsection{Training Dynamics and Coefficient Adaptation}
\label{app:training_curves}
\label{app:alpha_analysis}

\begin{figure}[ht]
    \centering
    \includegraphics[width=\textwidth]{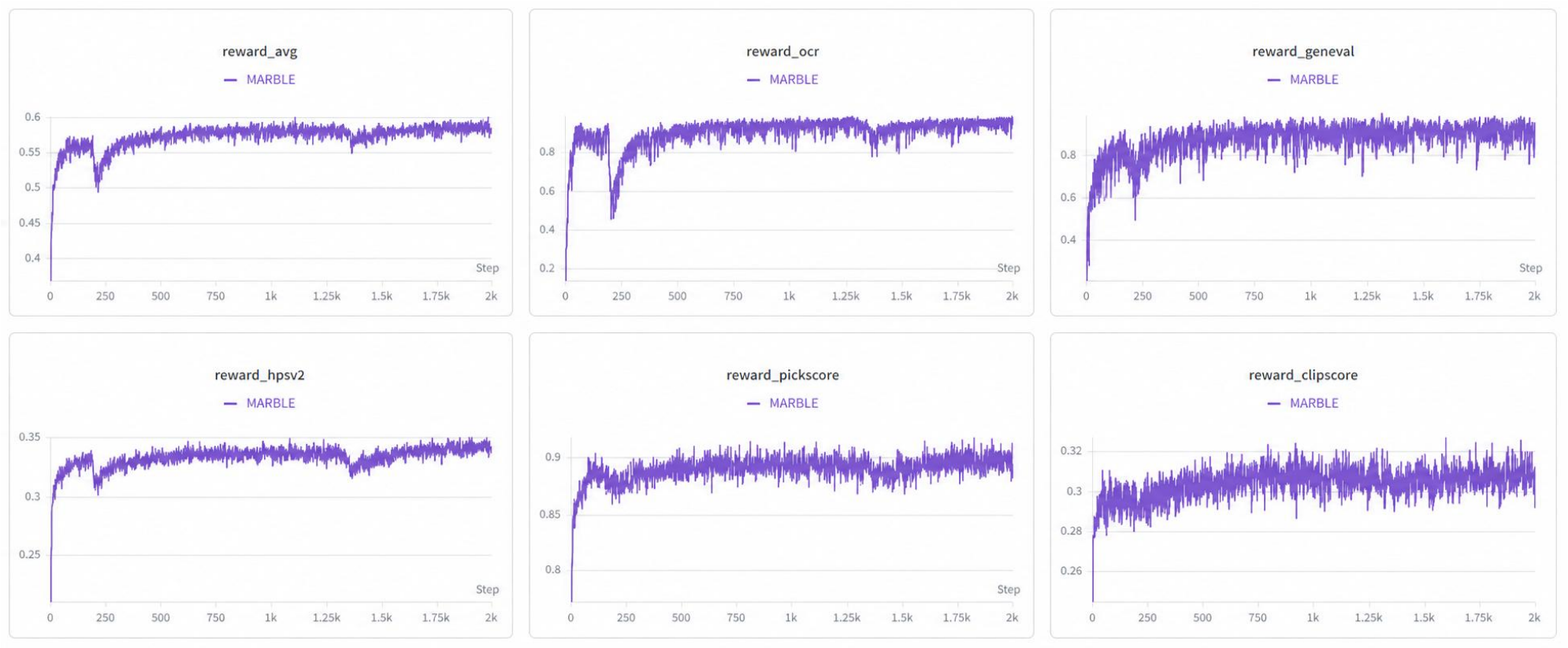}
    \caption{
    Training curves of \ours{} across the five optimized rewards. The curves show that all rewards continue improving during training, including both general image-quality rewards and specialist rewards.
    }
    \label{fig:marble_curves}
\end{figure}

\begin{figure}[ht]
    \centering
    \includegraphics[width=\textwidth]{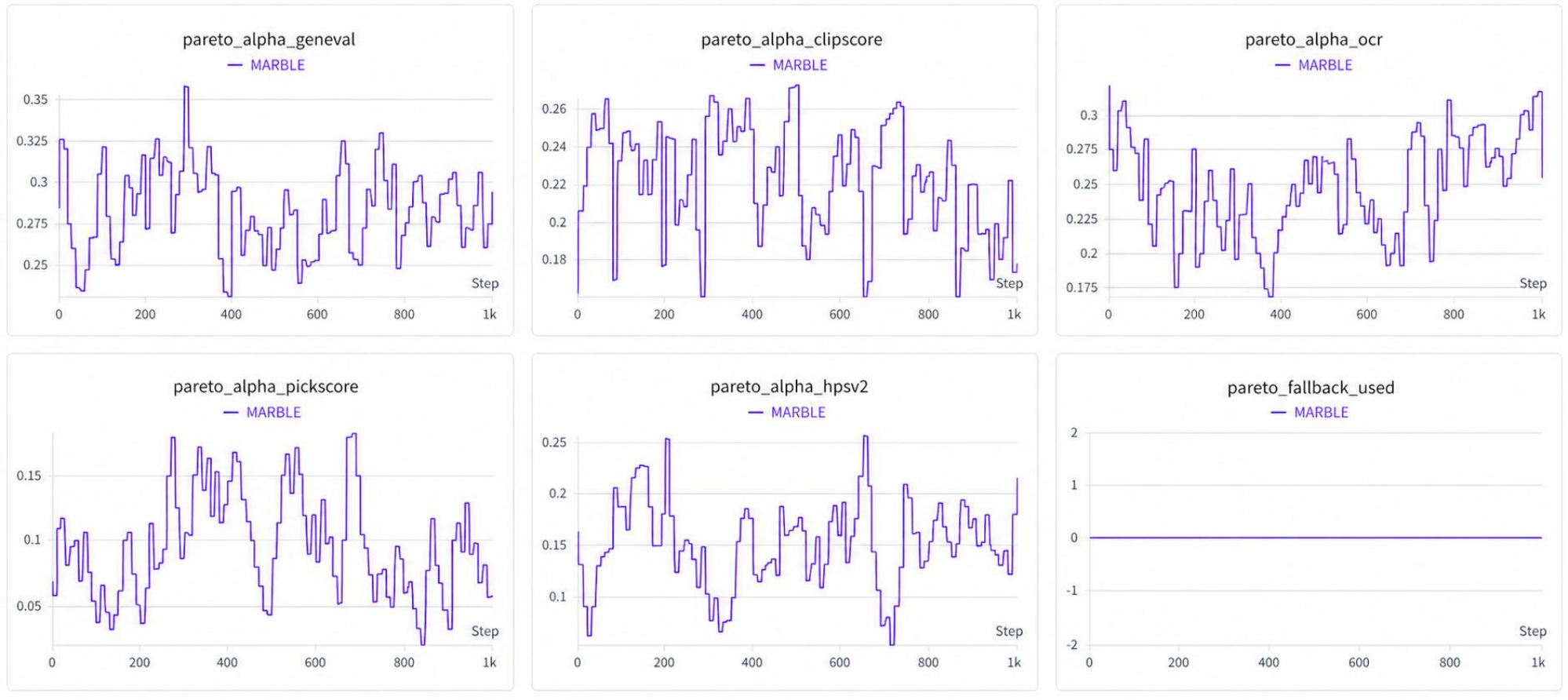}
    \caption{
    Dynamics of the smoothed balancing coefficients $\bar\alpha_{t,k}$ during training. The uniform baseline for five rewards is $0.2$. The learned coefficients do not directly track raw reward values, but show different allocation patterns across rewards with different optimization difficulty. The curve \texttt{pareto\_fallback\_used} counts how often the clamp condition discussed in Section~\ref{sec:amortized} is triggered.
    }
    \label{fig:alpha}
\end{figure}

\begin{figure}[ht]
    \centering
    \includegraphics[width=\textwidth]{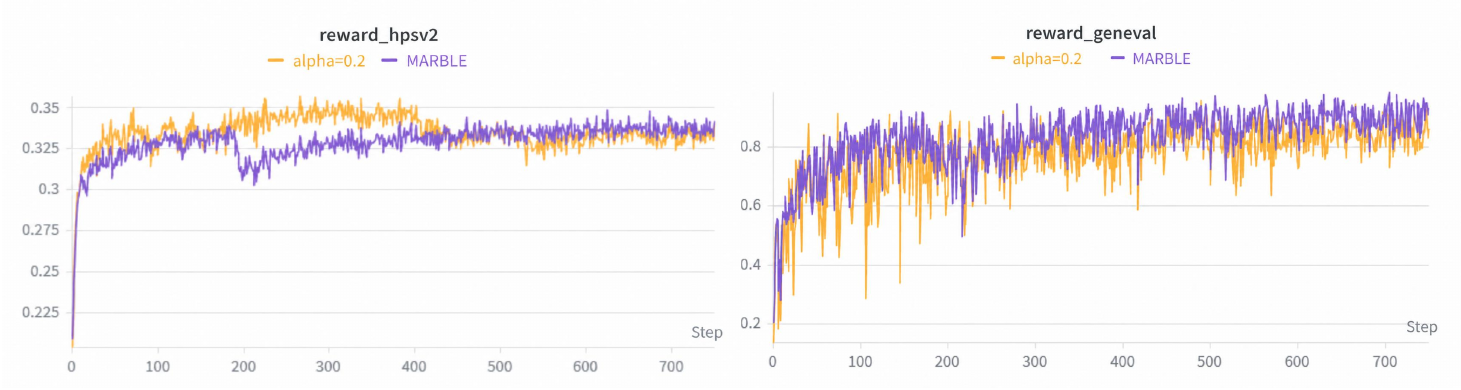}
    \caption{
    Learning-curve comparison between fixed uniform coefficients $\alpha_k=0.2$ and full \ours{}. Uniform coefficients make fast early progress on HPSv2, a broad image-quality reward, but converge more slowly and to a lower final score on GenEval, a harder specialist reward.
    }
    \label{fig:0.2_Marble}
\end{figure}

We complement the Equal-$\alpha$ ablation in Section~\ref{sec:ablation} with a closer view of how the smoothed coefficients $\bar\alpha_{t,k}$ evolve during training.

\paragraph{Uniform $\alpha=0.2$ is competitive on easy rewards but loses on specialists.}
Figure~\ref{fig:0.2_Marble} shows the per-reward learning curves under fixed $\alpha_k=0.2$ and under \ours{}. On HPSv2, an easy image-quality reward, fixed $\alpha_k=0.2$ trains slightly faster than \ours{} in the early iterations, while \ours{} converges to a higher final value. On GenEval, a harder specialist reward, fixed $\alpha_k=0.2$ trains more slowly and ends at a lower score than \ours{}, matching the Equal-$\alpha$ row of Table~\ref{tab:ablation}.

\paragraph{$\bar\alpha_{t,k}$ is related to optimization difficulty.}
As shown in Figure~\ref{fig:marble_curves} and~\ref{fig:alpha}, we do not observe a direct relationship between $\bar\alpha_{t,k}$ and the corresponding raw reward value. Instead, $\bar\alpha_{t,k}$ appears to be related, to some extent, to how hard each reward is to optimize (Figure~\ref{fig:alpha}). On HPSv2, which the base model already largely satisfies, $\bar\alpha_{t,k}$ stays below the uniform $0.2$ baseline for most of training. On GenEval, which demands precise compositional correctness, $\bar\alpha_{t,k}$ often rises to around $0.3$. Figure~\ref{fig:alpha} also reports \texttt{pareto\_fallback\_used}, the number of times the clamp condition in Section~\ref{sec:amortized} is triggered; this value stays at zero throughout training, indicating that the clamp is not activated in the plotted run. Together, these observations suggest that the learned coefficients can shift optimization emphasis away from easier rewards and toward harder specialist rewards, helping the five rewards progress more evenly during training (Figure~\ref{fig:marble_curves}).

\subsection{Amortization Interval}
\label{app:amortization_interval}

\begin{table*}[ht]
    \centering
    \renewcommand{\arraystretch}{1.1}
    \vspace{-2mm}
    \caption{\textbf{Sensitivity to the amortization interval.} We compare different values of $N$, which controls how often the full harmonization procedure refreshes $\alpha^*$. Smaller intervals refresh the balancing coefficients more frequently, while larger intervals reuse older coefficients for more update steps.}
    \label{tab:amortization_interval}
    \resizebox{\linewidth}{!}{
        \begin{tabular}{lcccccccc}
            \toprule
            \multirow{2}{*}{\textbf{Interval $N$}} 
            & \multicolumn{2}{c}{\textbf{Rule-Based}} 
            & \multicolumn{6}{c}{\textbf{Model-Based}} \\
            \cmidrule(lr){2-3} \cmidrule(lr){4-9}
            & \textbf{GenEval} 
            & \textbf{OCR} 
            & \textbf{PickScore} 
            & \textbf{CLIPScore} 
            & \textbf{HPSv2.1} 
            & \textbf{Aesthetic} 
            & \textbf{ImgRwd} 
            & \textbf{UniRwd} \\
            \midrule
            $N=5$  & 0.92 & \textbf{0.96} & 22.51 & 0.280 & 0.350 & 6.56 & 1.48 & 3.43 \\
            $N=10$ & \textbf{0.94} & \textbf{0.96} & \textbf{22.83} & \textbf{0.286} & \textbf{0.355} & \textbf{6.59} & \textbf{1.53} & \textbf{3.52} \\
            $N=20$ & 0.92 & 0.93 & 21.97 & 0.277 & 0.332 & 6.31 & 1.37 & 3.28 \\
            \bottomrule
        \end{tabular}
    }
    \vspace{-1mm}
\end{table*}

The amortization interval $N$ controls how often the full harmonization procedure recomputes $\alpha^*$ before the cached coefficients are reused for single-backward updates. A smaller $N$ refreshes the balancing coefficients more frequently, making the update direction more responsive to the current gradient geometry but increasing the average training overhead. A larger $N$ lowers this overhead, but it also reuses potentially stale coefficients for more steps. In our sensitivity runs, $N{=}5$ and $N{=}10$ give similar performance where $N{=}10$ performs only slightly better, suggesting that the EMA-smoothed coefficients are stable enough to be reused over a moderate window. Between these two settings, we choose $N{=}10$ because it refreshes the coefficients less frequently and is therefore slightly faster than $N{=}5$ while maintaining comparable performance. However, increasing the interval to $N{=}20$ leads to a performance drop, indicating that overly infrequent refreshes can make the cached coefficients lag behind the changing reward-gradient geometry. We therefore use $N{=}10$ in the main experiments as a practical middle ground between frequent coefficient refresh and lower computational overhead.

\subsection{EMA Decay for Coefficient Smoothing}
\label{app:ema_decay}

\begin{table*}[t]
    \centering
    \renewcommand{\arraystretch}{1.1}
    \vspace{-2mm}
    \caption{\textbf{EMA decay for coefficient smoothing.} The default setting $\rho=0.7$ is used in the main experiments; the remaining rows list decay values considered around this default.}
    \resizebox{\linewidth}{!}{
        \begin{tabular}{lcccccccc}
            \toprule
            \multirow{2}{*}{\textbf{EMA decay}} 
            & \multicolumn{2}{c}{\textbf{Rule-Based}} 
            & \multicolumn{6}{c}{\textbf{Model-Based}} \\
            \cmidrule(lr){2-3} \cmidrule(lr){4-9}
            & \textbf{GenEval} 
            & \textbf{OCR} 
            & \textbf{PickScore} 
            & \textbf{CLIPScore} 
            & \textbf{HPSv2.1} 
            & \textbf{Aesthetic} 
            & \textbf{ImgRwd} 
            & \textbf{UniRwd} \\
            \midrule
            $\rho = 0.1$ 
            & 0.86 & 0.80 & 21.52 & 0.261 & 0.292 & 5.84 & 1.22 & 2.98 \\
            $\rho = 0.3$ 
            & 0.88 & 0.84 & 21.76 & 0.266 & 0.312 & 6.03 & 1.27 & 3.04 \\
            $\rho = 0.5$ 
            & 0.93 & 0.95 & 22.02 & 0.276 & 0.340 & 6.14 & 1.48 & 3.43 \\
            $\rho = 0.7$
            & \textbf{0.94} & \textbf{0.96} & \textbf{22.83} & \textbf{0.286} & \textbf{0.355} & \textbf{6.59} & \textbf{1.53} & \textbf{3.52} \\
            $\rho = 0.9$ 
            & 0.90 & 0.89 & 22.14 & 0.272 & 0.342 & 6.26 & 1.47 & 3.40 \\
            \bottomrule
        \end{tabular}
    }
    \vspace{-1mm}
    \label{tab:ablation_ema}
\end{table*}

The EMA decay $\rho$ controls the adaptivity and stability in coefficient smoothing. Smaller values make the coefficients more responsive to the current gradient geometry, but also more sensitive to mini-batch noise, leading to larger fluctuations in the reward curves. Larger values produce smoother coefficient trajectories, but may become overly inertial and adapt slowly when the relative difficulty of rewards changes during training. We evaluate $\rho \in \{0.1,0.3,0.5,0.7,0.9\}$, with the quantitative results reported in Table~\ref{tab:ablation_ema}, the training curves shown in Fig.~\ref{fig:EMA}, and qualitative visualizations provided in Fig.~\ref{fig:EMA_viz}. Across these evaluations, $\rho=0.7$ provides the best overall performance: it maintains stable optimization, achieves strong performance, and produces visually more faithful and coherent samples. We therefore use $\rho=0.7$ as the default setting in all main experiments.

\begin{figure}[ht]
    \centering
    \includegraphics[width=\textwidth]{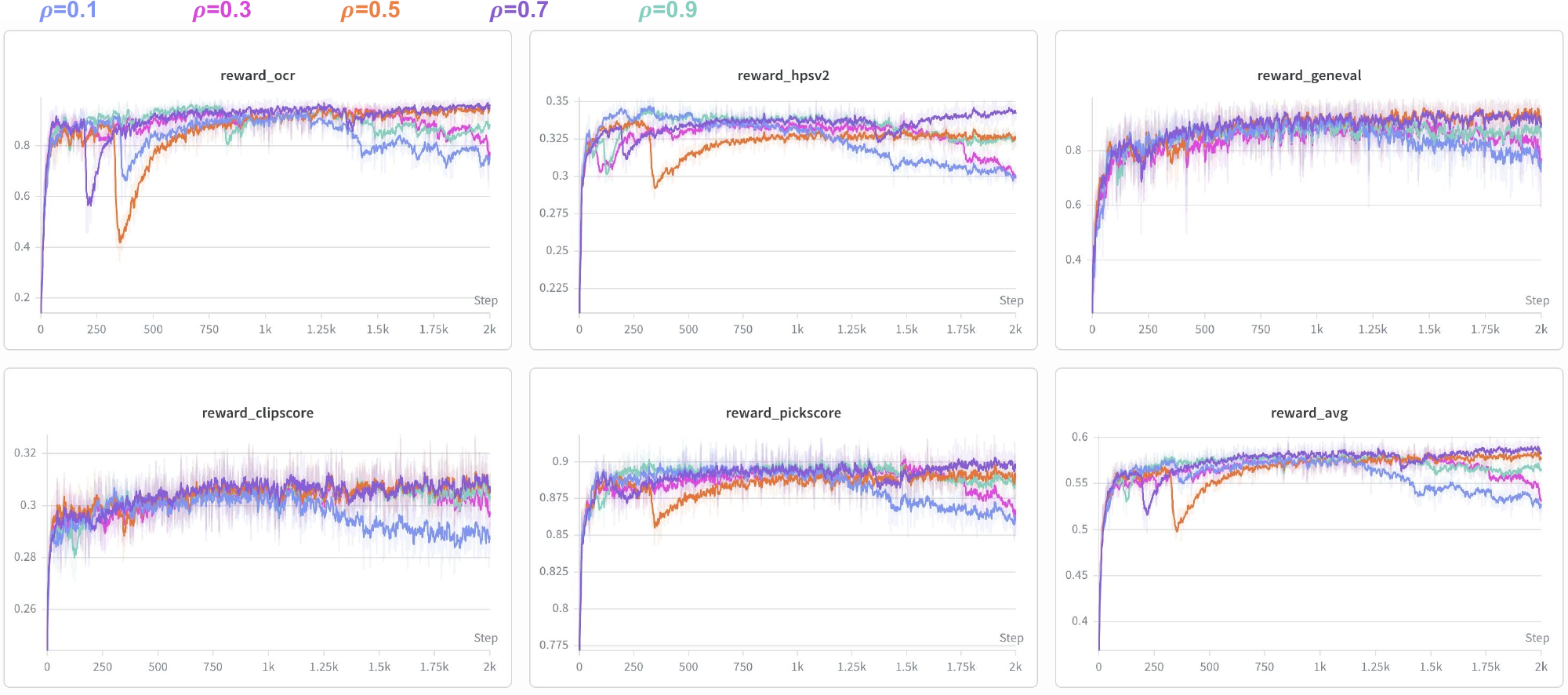}
    \caption{Sensitivity to EMA decay $\rho$. $\rho=0.7$ achieves the best overall performance.
    }
    \label{fig:EMA}
\end{figure}

\begin{figure}[ht]
    \centering
    \includegraphics[width=\textwidth]{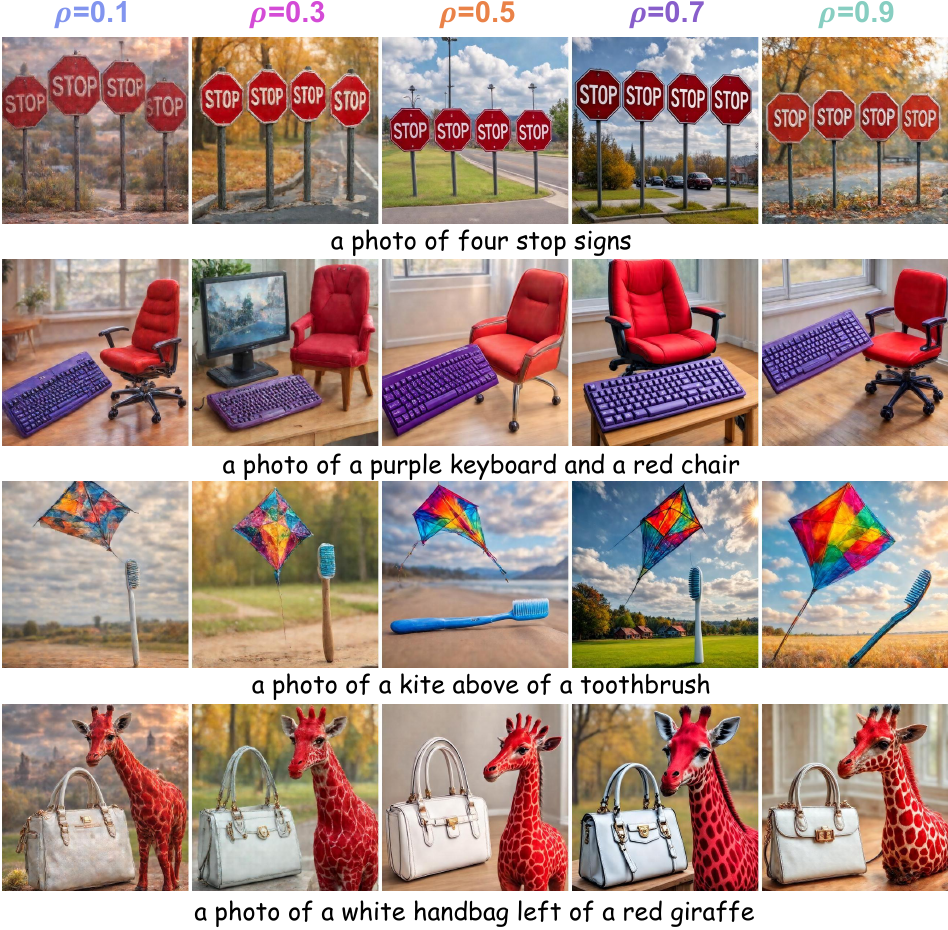}    \caption{Qualitative comparison of different EMA decay values $\rho$.}
    \label{fig:EMA_viz}
\end{figure}

\subsection{Alternative Heuristic Strategies}
\label{app:alternative_strategies}
\begin{table*}[t]
    \centering
    \renewcommand{\arraystretch}{1.1}
    \vspace{-2mm}
    \caption{\textbf{Alternative reward-balancing strategies.} We compare \ours{} with several heuristic strategies for multi-reward diffusion RL. All variants use the same 5-reward setup and training budget.}
    \resizebox{\linewidth}{!}{
        \begin{tabular}{lcccccccc}
            \toprule
            \multirow{2}{*}{\textbf{Method}} 
            & \multicolumn{2}{c}{\textbf{Rule-Based}} 
            & \multicolumn{6}{c}{\textbf{Model-Based}} \\
            \cmidrule(lr){2-3} \cmidrule(lr){4-9}
            & \textbf{GenEval} 
            & \textbf{OCR} 
            & \textbf{PickScore} 
            & \textbf{CLIPScore} 
            & \textbf{HPSv2.1} 
            & \textbf{Aesthetic} 
            & \textbf{ImgRwd} 
            & \textbf{UniRwd} \\
            \midrule
            \ours{} & \textbf{0.94} & \textbf{0.96} & \textbf{22.83} & \textbf{0.286} & \textbf{0.355} & \textbf{6.59} & \textbf{1.53} & \textbf{3.52} \\
            \midrule
            Reward Dropout 
            & 0.80 & 0.61 & 21.35 & 0.267 & 0.270 & 5.89 & 0.73 & 2.88 \\
            Reward Weighting $(1{:}1{:}1{:}2{:}2)$
            & 0.90 & 0.92 & 21.06 & 0.271 & 0.285 & 5.91 & 1.19 & 3.13 \\
            \bottomrule
        \end{tabular}
    }
    \vspace{-1mm}
    \label{tab:alternative_strategies}
\end{table*}

We also examine several simple heuristic strategies for balancing multiple rewards in diffusion RL, as shown in Table~\ref{tab:alternative_strategies}. These include using fixed uniform coefficients and increasing the scalar weights of more challenging specialist rewards, such as OCR and GenEval, in the weighted-sum objective. However, our experiments show that these heuristic strategies fail to achieve simultaneous improvements across all reward dimensions and can lead to performance degradation on several metrics. This suggests that reward balancing in multi-reward diffusion RL cannot be reliably addressed by manual reward reweighting or coarse reward-level heuristics, further demonstrating the effectiveness of \ours{}.

\subsection{Human Preference Evaluation and Metric Correlations}
\label{app:user_study}

Evaluating image generation quality requires considering multiple complementary criteria rather than relying on a single automatic proxy. As shown in Table~\ref{tab:main}, different metrics do not always rank methods consistently: DiffusionNFT$^\dagger$ obtains higher PickScore and CLIPScore, whereas \ours{} achieves the best Composite score and performs better on several broader quality- and preference-oriented metrics. This discrepancy reflects the fact that automatic metrics capture different aspects of generation quality. 

To obtain a more comprehensive assessment of human preference, we conduct a blind rating-based user study. For each method, we randomly sample $30$ generated images. All images are anonymized, randomly shuffled, and shown with their corresponding prompts. A total of $20$ anonymous participants who are unrelated to the project independently score each image on a $1$--$5$ scale along two axes: text-image alignment and image quality. Higher scores indicate better perceived performance. We report the average score over all ratings in Table~\ref{tab:user_study}.

\begin{table*}[ht]
\centering
\caption{User study results. Participants score anonymized and randomly shuffled images on a $1$--$5$ scale for text-image alignment and image quality. Higher scores are better.}
\label{tab:user_study}
\begin{tabularx}{\textwidth}{
    >{\raggedright\arraybackslash}X
    >{\centering\arraybackslash}X
    >{\centering\arraybackslash}X
}
\toprule
Method & Text-image alignment $\uparrow$ & Image quality $\uparrow$ \\
\midrule
DiffusionNFT$^\ddagger$ & 3.60 & 2.79 \\
DiffusionNFT$^\dagger$ & 4.26 & 3.58 \\
\ours{} & \textbf{4.63} & \textbf{4.41} \\
\bottomrule
\end{tabularx}
\end{table*}

\ours{} receives the highest average score on both text-image alignment and image quality. We do not claim statistical significance from this study; rather, we use it as a complementary human-centered evaluation to automatic metrics. The results suggest that the lower PickScore and CLIPScore of \ours{} do not correspond to lower human-rated quality in this setting. Instead, the user study is more consistent with the broader set of automatic metrics, including the Composite score, supporting the need to evaluate image generation quality with multiple complementary criteria.

\paragraph{Automatic metrics and human judgment.}
We further analyze how different automatic metrics align with human preference. Specifically, we compute the Pearson correlation between each image-quality-related metric and the two human-rated axes, namely image quality and text-image alignment, across all images in the user study. Table~\ref{tab:user_study_corr} reports the results.

The correlation analysis shows that holistic preference metrics, including HPSv2.1, Aesthetic Score, UniReward, and ImageReward, are positively correlated with both human-rated axes, with HPSv2.1 showing the strongest agreement. In contrast, PickScore and CLIPScore are weaker predictors of human ratings. This result further indicates that no single automatic metric is sufficient to characterize human-perceived generation quality. Therefore, the PickScore/CLIPScore advantage of DiffusionNFT$^\dagger$ should be interpreted as a metric-specific difference rather than a definitive indication of superior perceptual quality.

The qualitative visualizations in Figure~\ref{fig:supp1}--\ref{fig:supp4} provide further evidence for this conclusion. Across diverse prompts, \ours{} better preserves fine-grained requirements such as text rendering, attribute binding, spatial layout, and counting, while maintaining sharper and more coherent visual quality. In contrast, the weighted-sum baseline often fails to improve all aspects simultaneously, and DiffusionNFT$^\dagger$ sometimes produces less sharp or less detailed images despite its strong proxy scores. Taken together, the quantitative results in Table~\ref{tab:main}, the qualitative comparisons in Figure~\ref{fig:supp1}--\ref{fig:supp4}, and the user study consistently show that \ours{} achieves the best overall balance between automatic metrics, visual quality, and human preference.

\begin{table}[ht]
\centering
\caption{\textbf{Pearson correlation between automatic metrics and human ratings}, computed on all images of the user study (sorted by correlation with image quality). Higher (more positive) values indicate stronger agreement with the human axis.}
\label{tab:user_study_corr}
\small
\setlength{\tabcolsep}{6pt}
\renewcommand{\arraystretch}{1.1}
\begin{tabularx}{\textwidth}{l >{\centering\arraybackslash}X >{\centering\arraybackslash}X}
\toprule
\textbf{Metric} & \textbf{r with image quality} & \textbf{r with text-image alignment} \\
\midrule
HPSv2.1     & +0.66 & +0.56 \\
Aesthetic   & +0.36 & +0.33 \\
UniReward   & +0.25 & +0.39 \\
ImageReward & +0.19 & +0.35 \\
PickScore   & +0.17 & +0.26 \\
CLIPScore   & +0.00 & +0.09 \\
\bottomrule
\end{tabularx}
\end{table}

\section{Additional Implementation Details}
\label{app:implementation}

We provide full implementation details for reproducing \ours{} on the DiffusionNFT codebase with distributed training.

In distributed data-parallel (DDP) training, each GPU processes a different data shard and gradients are averaged across ranks via \texttt{AllReduce} before the optimizer step. With \ours{}, the standard DDP gradient synchronization is insufficient because we need per-reward gradients \emph{before} combining them via gradient harmonization.

The synchronization protocol proceeds as follows for each training iteration:

\begin{algorithm}[h]
\caption{\ours{} DDP Synchronization}
\label{alg:ddp}
\begin{algorithmic}[1]
\REQUIRE Model with DDP wrapper, $K$ per-reward losses $\{\ell_k\}$
\STATE Initialize gradient storage: $G = [\;]$ \hfill \textit{// list of $K$ flattened gradient vectors}
\FOR{$k = 1$ to $K$}
    \STATE \textbf{with} model.\texttt{no\_sync()}: \hfill \textit{// suppress DDP AllReduce}
    \STATE \quad $\ell_k$.\texttt{backward(retain\_graph=(k < K))}
    \STATE \quad $g_k \leftarrow$ \texttt{flatten\_grads(model)} \hfill \textit{// extract and flatten trainable .grad}
    \STATE \quad \texttt{zero\_grads(model)}
    \STATE \quad $G$.\texttt{append}($g_k$)
\ENDFOR
\FOR{$k = 1$ to $K$}
    \STATE \texttt{dist.all\_reduce}($G[k]$, op=\texttt{AVG}) \hfill \textit{// synchronize across ranks}
\ENDFOR
\STATE All ranks now have identical $\{g_k\}_{k=1}^K$
\STATE Solve harmonization locally $\to$ identical $\alpha^*$ on every rank
\STATE $d^* \leftarrow \sum_k \alpha_k^* \hat{g}_k$
\STATE \texttt{unflatten\_to\_grad}($d^*$, model) \hfill \textit{// restore to parameter .grad}
\STATE Optimizer step proceeds as normal
\end{algorithmic}
\end{algorithm}

\paragraph{Why \texttt{no\_sync()} is necessary.}
Without \texttt{no\_sync()}, DDP would trigger an \texttt{AllReduce} on \emph{every} backward call. Since we call backward $K$ times (once per reward), this would: (1) average gradients prematurely before we can extract per-reward gradients, and (2) incur $K$ unnecessary collective operations. By wrapping each backward in \texttt{no\_sync()}, we defer synchronization to our explicit \texttt{all\_reduce} calls, where we synchronize the already-extracted per-reward gradient vectors.

\paragraph{\texttt{retain\_graph} handling.}
The first $K{-}1$ backward passes use \texttt{retain\_graph=True} because all per-reward losses share the same forward computation graph. The last backward pass uses \texttt{retain\_graph=False} to release the computation graph and free memory.

\section{Future Work}
\label{app:future_work}

\paragraph{Scalability to More Rewards.}
This work studies reward balancing with five reward dimensions, while scaling to a larger and more diverse set of rewards remains an important future direction. In future work, we aim to further investigate the scalability of \ours{} and develop more effective balancing strategies for handling increasingly diverse and potentially conflicting reward signals.

\paragraph{Extension to Video and World Models.}
Extending \ours{} to video generation is another promising direction, especially in the context of generative world modeling. Video models require jointly optimizing a broader set of objectives, including visual fidelity, temporal consistency, motion realism, and physical plausibility. These heterogeneous objectives are central to world models, which require not only high-quality generation but also coherent dynamics and plausible long-horizon evolution. We believe that \ours{} can serve as a promising optimization strategy for addressing these tasks.

\begin{figure}[ht]
    \centering
    \includegraphics[width=.9\textwidth]{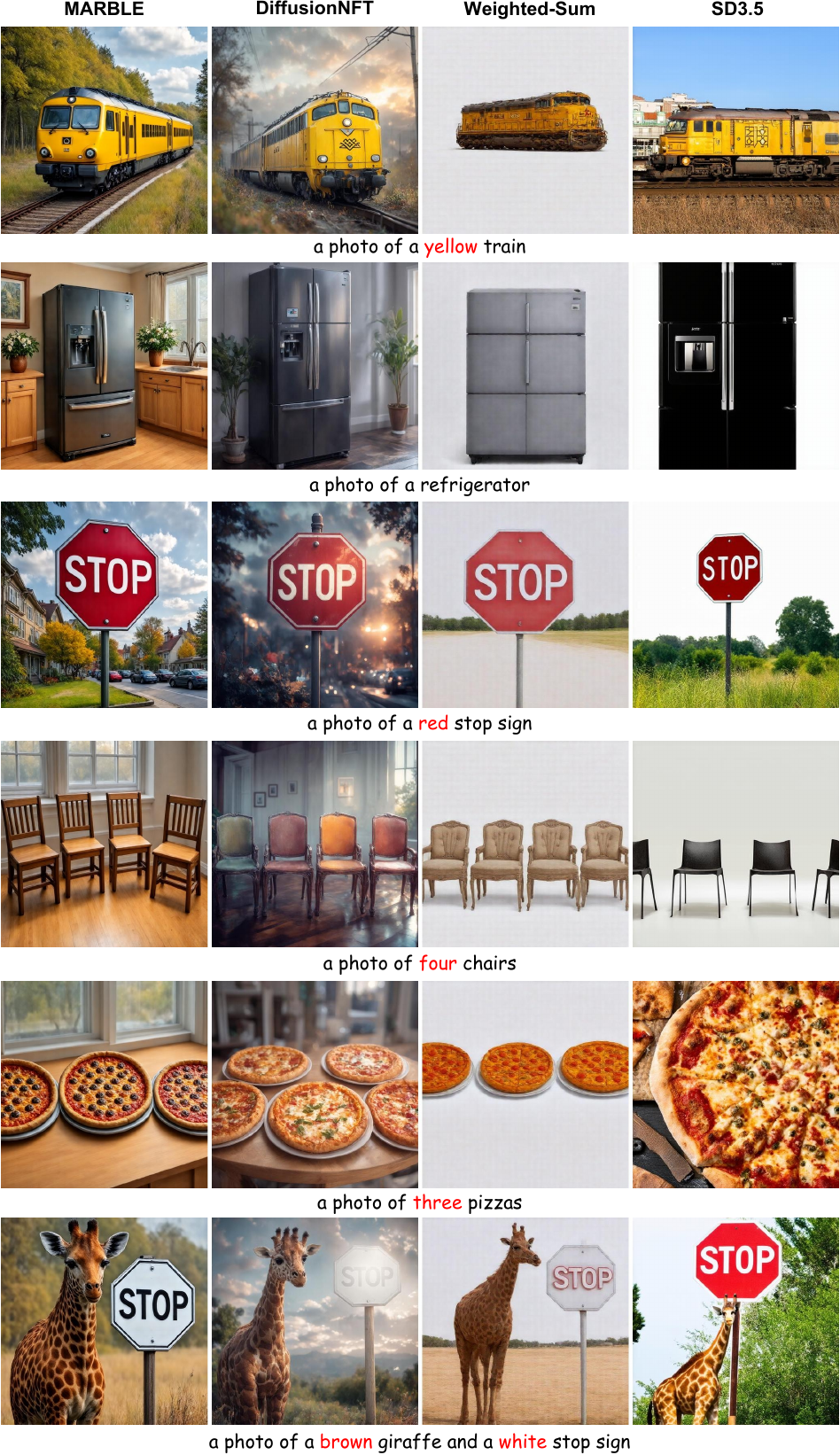}
    \caption{
    Additional qualitative comparisons.
    }
    \label{fig:supp1}
\end{figure}

\begin{figure}[ht]
    \centering
    \includegraphics[width=.9\textwidth]{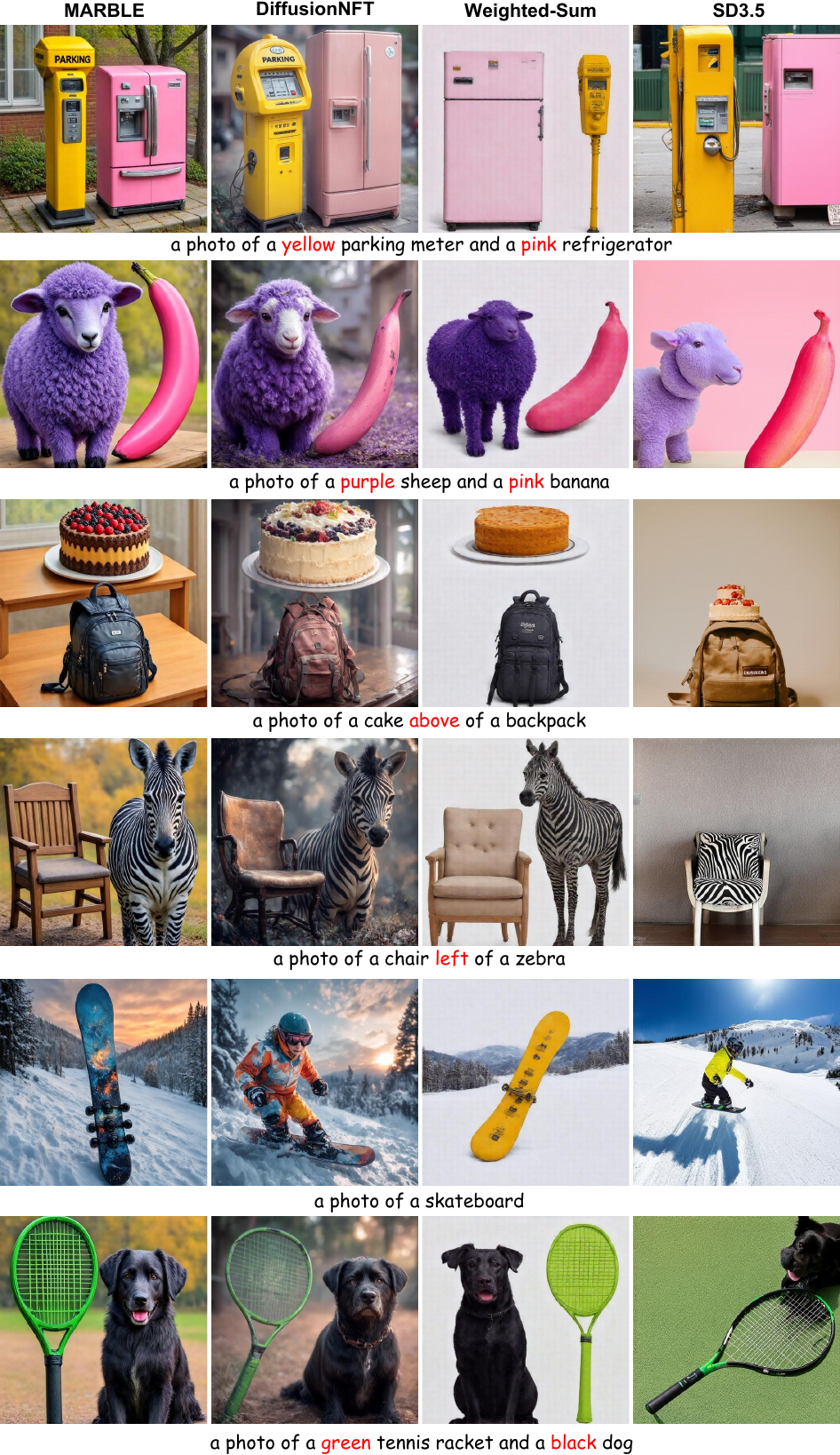}
    \caption{
    Additional qualitative comparisons.
    }
    \label{fig:supp2}
\end{figure}

\begin{figure}[ht]
    \centering
    \includegraphics[width=.9\textwidth]{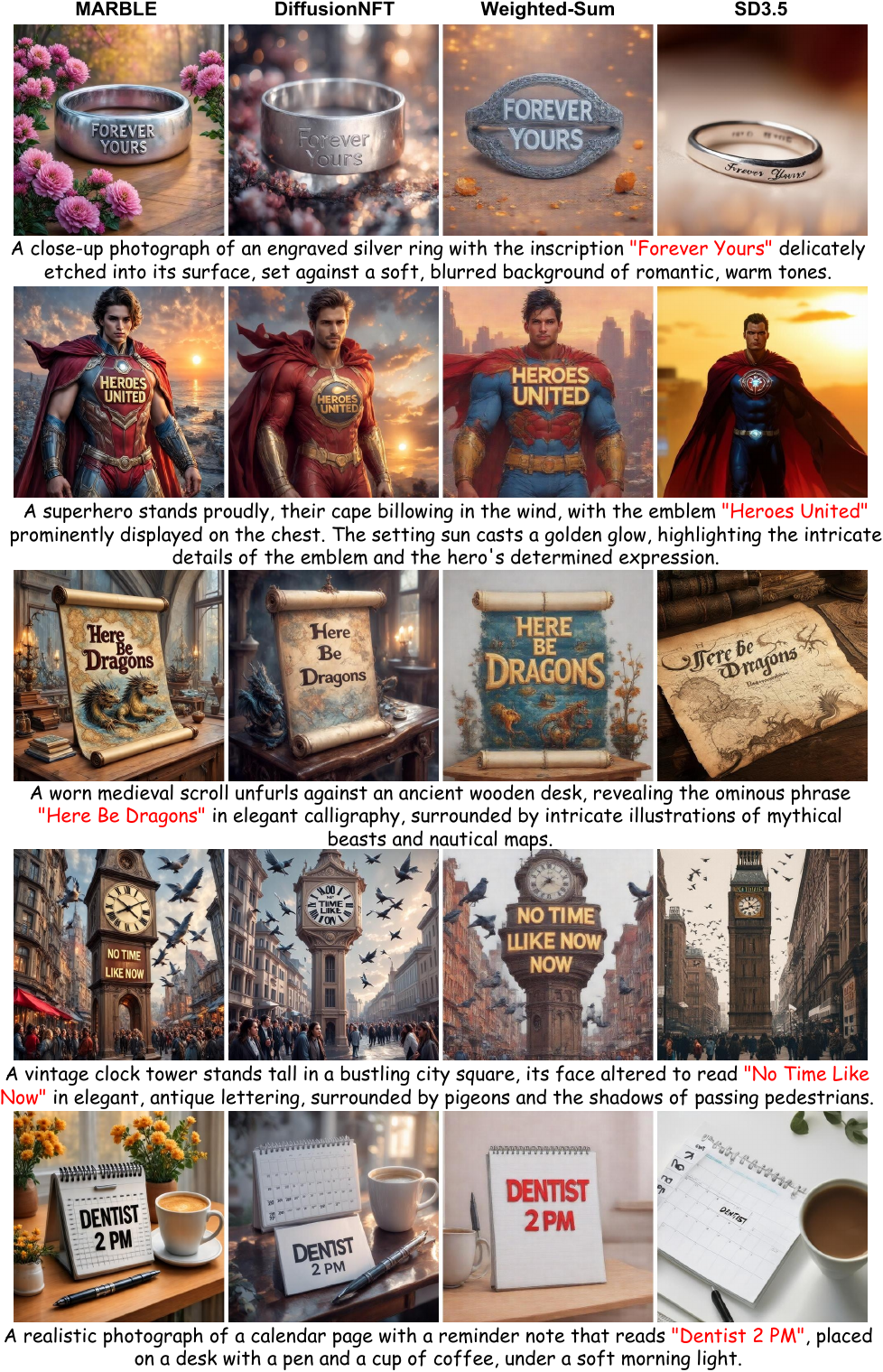}
    \caption{
    Additional qualitative comparisons.
    }
    \label{fig:supp3}
\end{figure}

\begin{figure}[ht]
    \centering
    \includegraphics[width=.85\textwidth]{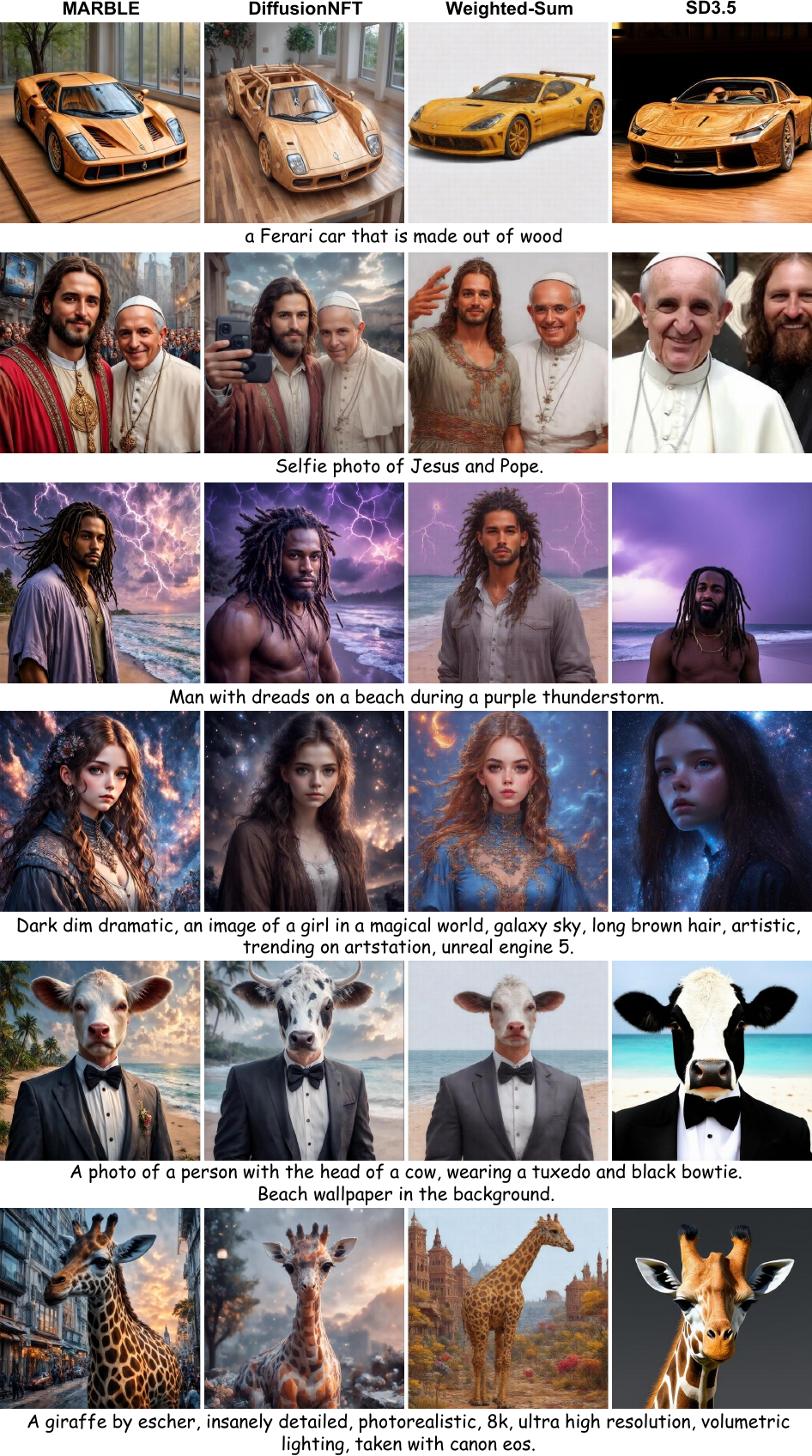}
    \caption{
    Additional qualitative comparisons.
    }
    \label{fig:supp4}
\end{figure}

\end{document}